\documentclass[12pt]{article}

\usepackage{amssymb}
\usepackage{amsthm,amsmath}
\usepackage{natbib}
\RequirePackage[colorlinks,citecolor=blue,urlcolor=blue]{hyperref}
\usepackage{graphicx}
\usepackage[colorinlistoftodos]{todonotes}
\usepackage{eucal} 
\usepackage{outlines}
\usepackage{mathtools}
\usepackage{booktabs}
\usepackage{multirow}
\usepackage{ragged2e}
\usepackage{ulem}
\usepackage{color}
\usepackage{rotating}
\usepackage{xr}
\usepackage{algorithm}
\usepackage{algorithmic}

\usepackage{color,soul}
\usepackage{subcaption}
\usepackage{bbm}
\usepackage{booktabs}
\usepackage{multirow}
\usepackage{pdflscape}
\usepackage{array}
\usepackage[font=small]{caption}

\newcolumntype{H}{>{\setbox0=\hbox\bgroup}c<{\egroup}@{}}

\newtheorem{theorem}{Theorem}
\newtheorem{lemma}{Lemma}

\newtheorem{definition}{Definition}

\DeclareMathOperator*{\argmin}{arg\,min}

\DeclareMathOperator*{\Var}{Var}
\DeclareMathOperator{\E}{\mathbb{E}}

\usepackage{etoolbox}
\usepackage{setspace}

\title{Ensembled sparse-input hierarchical networks for high-dimensional datasets}
\author{Jean Feng and Noah Simon}
\date{}

\begin{document}
\maketitle

\begin{abstract}
Neural networks have seen limited use in prediction for high-dimensional data with small sample sizes, because they tend to overfit and require tuning many more hyperparameters than existing off-the-shelf machine learning methods.
With small modifications to the network architecture and training procedure, we show that dense neural networks can be a practical data analysis tool in these settings.
The proposed method, \textbf{E}nsemble by \textbf{A}veraging \textbf{S}parse-\textbf{I}nput hi\textbf{ER}archical networks (EASIER-net), appropriately prunes the network structure by tuning only two $L_1$-penalty parameters, one that controls the input sparsity and another that controls the number of hidden layers and nodes.
The method selects variables from the true support if the irrelevant covariates are only weakly correlated with the response; otherwise, it exhibits a grouping effect, where strongly correlated covariates are selected at similar rates.
On a collection of real-world datasets with different sizes, EASIER-net selected network architectures in a data-adaptive manner and achieved higher prediction accuracy than off-the-shelf methods on average.
\end{abstract}


\section{Introduction}

Deep neural networks are highly modular and expressive nonparametric models that can learn complicated relationships between their inputs and outputs.
Although they are state-of-the-art in many complex prediction problems where large amounts of data are available, their utility for analyzing datasets with few observations has been limited, particularly when the number of dimensions is high and the signal to noise ratio is low.
There are two reasons why data analysts usually dismiss deep neural networks in these settings.
First, they easily overfit to the training data.
Second, the data analyst needs to tune many hyperparameters, including the number of hidden layers, the number of hidden nodes per layer, regularization parameters, learning rates for the optimization algorithm, among others \citep{Bengio2012-zq}.

Nevertheless, the assumption that deep learning is too brittle or burdensome in these settings is worth revisiting.
Many methods have been developed to reduce the generalization error of deep neural networks \citep{Goodfellow2016-sw}, such as the use of rectified linear units \citep{Glorot2011-pd}, stochastic gradient descent, regularization, and ensemble methods.
In fact, \citet{Zhang2017-ws} has urged the research community to rethink generalization error for neural networks, since these models were able to both achieve low generalization error \textit{and} memorize the training data.
In addition, neural networks are easier than ever to implement and train.
Many powerful automatic differentiation packages are freely available (e.g. {\tt Tensorflow} \citep{Abadi2016-zh} and {\tt PyTorch} \citep{Paszke2019-tz}) and computing power has grown exponentially in the past few decades.

Although most analyses of deep learning have been performed in large image datasets, there is now a growing body of evidence that neural networks can also be used for smaller sample sizes.
In an analysis of over a hundred UCI classification datasets, \citet{Olson2018-vs} showed that ensemble neural networks with zero hyperparameter tuning do only slightly worse than random forests on average.
An alternate approach pruned input weights using $L_1$ and $L_2$ penalties, resulting in sparse-input neural networks (SPINN), and outperformed random forests and the Lasso in certain cases \citep{Feng2019-nw}.
Finally, Bayesian neural networks have been used for variable selection in high-dimensional datasets \citep{Neal1996-at, Guyon2005-fp, Liang2018-ev}.

While these methods are promising, they are still insufficient.
We found that the performance of network ensembles in \citet{Olson2018-vs} heavily depends on the dataset and can be quite poor without any network structure learning.
SPINN requires tuning five or more hyperparameters, which is time consuming even with the help of modern hyperparameter optimization methods (see e.g., \citet{Snoek2012-zd}).
Bayesian networks even more computationally expensive than their non-Bayesian counterparts.

Our goal is to design a procedure that is competitive with off-the-shelf machine learning algorithms and  requires only a ``reasonable'' amount of hyperparameter tuning.
As many popular statistical procedures require tuning one or two hyperparameters (see e.g. \citet{Zou2005-wm} and \citet{Simon2013-tq}), we define tuning two hyperparameters as ``reasonable."
Training neural networks typically involves an outer search over hyperparameters such as the entire network structure and an inner training procedure for each candidate hyperparameter set.
Here, we design the outer search to tune two \textit{summary} features of the network structure that correspond to major sources of variation for generalization error: input sparsity and network size.
The inner procedure is then responsible for both selecting the specific network structure and fitting model parameters.

We propose a method for training a single network, \textbf{S}parse-\textbf{I}nput hi\textbf{ER}archical network (SIER-net), and its ensemble version, \textbf{E}nsemble by \textbf{A}veraging \textbf{S}parse-\textbf{I}nput hi\textbf{ER}archical networks (EASIER-net).
In SIER-net, the network is initialized with a fixed architecture with many hidden layers (deep), many hidden nodes per layer (wide), skip-connections from each layer to the output, and an input filter layer.
Network structure learning is performed solely by tuning two $L_1$ penalty parameters, one to modulate the number of selected inputs and another for the number of hidden nodes and layers.
Although many nodes and layers are included in the initial topology, these are perhaps better thought of as ``candidate'' nodes and layers: The combination of penalties and skip-connections allows the network to data-adaptively determine appropriate widths and depth (up to the maximum network size indicated).

By deriving probability bounds on the variable selection accuracy, we show that SIER-net is likely to select covariates in the true support that have high predictive value relative to the remaining covariates.
On the other hand, if the covariates are highly correlated, we show that it is impossible for \textit{any} procedure to accurately recover the true support.
Since support recovery is unrealistic in these settings, we show that ensembling these networks using EASIER-net lets us quantify the uncertainty of the variable selection procedure.
Moreover, in cases where correlated variables belong to a meaningful group (e.g. gene expression values from the same pathway), EASIER-net exhibits a grouping effect, where covariates in the same group are selected at similar rates.

We evaluate EASIER-net on a collection of regression and classification tasks from the UCI Machine Learning Repository, which were chosen to represent a variety of dataset sizes.
EASIER-net consistently achieves higher prediction accuracy than other neural network estimators.
In addition, EASIER-net achieves higher prediction accuracy on average compared to off-the-shelf machine learning methods and is competitive in computation time.

The paper is organized as follows.
Section~\ref{sec:related} discusses related work.
In Section~\ref{sec:method}, we define SIER-net and analyze its variable screening accuracy.
We then introduce our ensembled estimator and discuss the effects of ensembling using a Bayesian perspective.
Finally, we present empirical analyses of EASIER-net on both simulated and real data in Section~\ref{sec:empirical}.
Code for fitting EASIER-net is available at \url{https://github.com/jjfeng/easier_net}.

\section{Related Work}
\label{sec:related}

There is a growing body of literature on applying neural networks to high-dimensional datasets with low sample sizes.
Existing methods typically rely on some form of regularization to mitigate the neural network's tendency to overfit.
Sparsity-inducing penalties, like the lasso and group lasso, can be applied to the network weights to encourage feature selection and drastically improve prediction accuracy \citep{Scardapane2017-yj,Yoon2017-ao, Feng2019-nw}.
A similar idea in the Bayesian literature is to apply a prior over the model weights, such as through an ``Automatic Relevance Determination''  (ARD) prior \citep{MacKay1996-ey, Neal1996-at, Neal2006-jb} or a Bernoulli prior for the probability that a weight is included \citep{Liang2018-ev}.
An alternative idea inspired by the success of random forests is to regularize neural networks by ensembling \citep{Olson2018-vs}.
Even though these methods are promising, neural networks still have limited utility in these settings.
Many of these methods either require substantially more hyperparameter tuning or offer negligible gains in prediction accuracy compared to existing off-the-shelf machine learning methods.
This paper combines three simple ideas --- sparsity, skip-connections, and ensembling --- to improve prediction accuracy and decrease hyperparameter tuning requirements.

Applying sparsity-inducing penalties to perform network structure learning is a well-accepted practice \citep{Hanson1989-cj, LeCun1990-ng, Yu2012-ka, Han2015-wz, Wen2016-lo}.
In addition to lower computational and memory requirements, sparser networks often achieve lower generalization error \citep{Louizos2018-mx, Yoon2017-ao, Feng2019-nw}.
However, these penalization-based methods typically require tuning many penalty parameters and/or network structure hyperparameters.
In contrast, EASIER-net requires tuning only two $L_1$ penalty parameters and select an appropriate network structure for each candidate penalty parameter set.

Skip-connections are a classical technique for expanding the neural network model class to include simpler models like linear and logistic regression (e.g. Chapter 5 of \citet{Ripley1996-ub}).
This technique often improves prediction accuracy on high-dimensional datasets with small sample sizes by leveraging the performance of simple parametric models.
As such, \citet{Liang2018-ev} used skip-connections to combine Bayesian linear regression with Bayesian neural networks with a single hidden layer.
Our method EASIER-net can fit even deeper networks and adaptively chooses the depth using a Lasso penalty.
Although skip-connections have also been used to improve gradient propagation in ultra-deep neural networks \citep{He2016-qz, Huang2017-bt}, our architecture is designed for shallower networks with five or so hidden layers, which tend to perform well on smaller datasets.

Ensembling has been shown to improve the generalization error \citep{Olson2018-vs} and uncertainty quantification of neural networks \citep{Lakshminarayanan2017-wn}.
The benefit of ensembling has been explained from many perspectives, including its ability to reduce model variance \citep{Breiman1996-bf} and its similarity to Bayesian posterior inference and model averaging \citep{Bardsley2014-bp, Lu2017-kz, Pearce2020-ee, Wilson2020-jw}.
However, none of these works have discussed the interaction between variable selection procedures and ensembling.
In this work, we show that ensembling sparse-input neural networks better reflects the uncertainty of the true support and induces a ``grouping effect'' for correlated predictors, a feature commonly associated with the elastic net \citep{Zou2005-wm}.

\section{Method}
\label{sec:method}

\subsection{Notation}
Consider the usual prediction setup with covariates $\boldsymbol{x} = (x_1,...,x_d) \in \mathcal{X} \subseteq \mathbb{R}^d$ and response $y \in \mathcal{Y}$.
$\mathcal{Y}$ may be real-valued in the case of regression and categorical in the case of classification.
Suppose we observe $n$ independent and identically distributed (IID) observations, denoted $(\boldsymbol{x}_i, y_i)$ for $i = 1,...,n$.

The typical dense neural network with $L$ layers is defined as follows.
Let $d_l$ be the number of nodes in layer $l$ and $\boldsymbol{z}_l \in \mathbb{R}^{1 \times d_l}$ be the node values, where $l = 1$ denotes the input layer (i.e. $\boldsymbol{z}_1 = \boldsymbol{x}$) and $l = L$ denotes the output.
For simplicity, we suppose all hidden nodes use the same nonlinear activation function $\phi: \mathbb{R} \mapsto \mathbb{R}$.
Here, we set $\phi$ as the rectified linear unit $a \mapsto a_+$, though one could use alternative activation functions like the sigmoid and hyperbolic tangent.
Each non-output layer $l$ is associated with a weight matrix $W_{l} \in \mathbb{R}^{d_{l} \times d_{l + 1}}$ and bias vector $b_{l} \in \mathbb{R}^{1 \times d_{l + 1}}$ that is used to linearly combine its inputs.
The specified activation function is then applied to create the input for the subsequent layer.
For $l = 1,...,L - 2$, we define $z_{l + 1} = \phi\left(z_{l} W_{l} + b_{l} \right)$ where $\phi$ is applied element-wise.
The final model output is defined as $\boldsymbol{z}_{L} = \phi_{\text{out}}(\boldsymbol{z}_{L - 1} W_{L - 1} + \boldsymbol{b}_{L - 1})$.
For regression problems, $d_L = 1$ and $\phi_{\text{out}}: \mathbb{R}^{1 \times 1} \mapsto \mathbb{R}$ is typically the identity function; For classification problems, $d_L$ is the number of classes and $\phi_{\text{out}}: \mathbb{R}^{1 \times d_L} \mapsto [0,1]^{d_L}$ is typically the softmax function.

\subsection{Sparse-input hierarchical networks}

\begin{figure}
	\centering
	\begin{subfigure}{0.45\linewidth}
	\includegraphics[width=\linewidth]{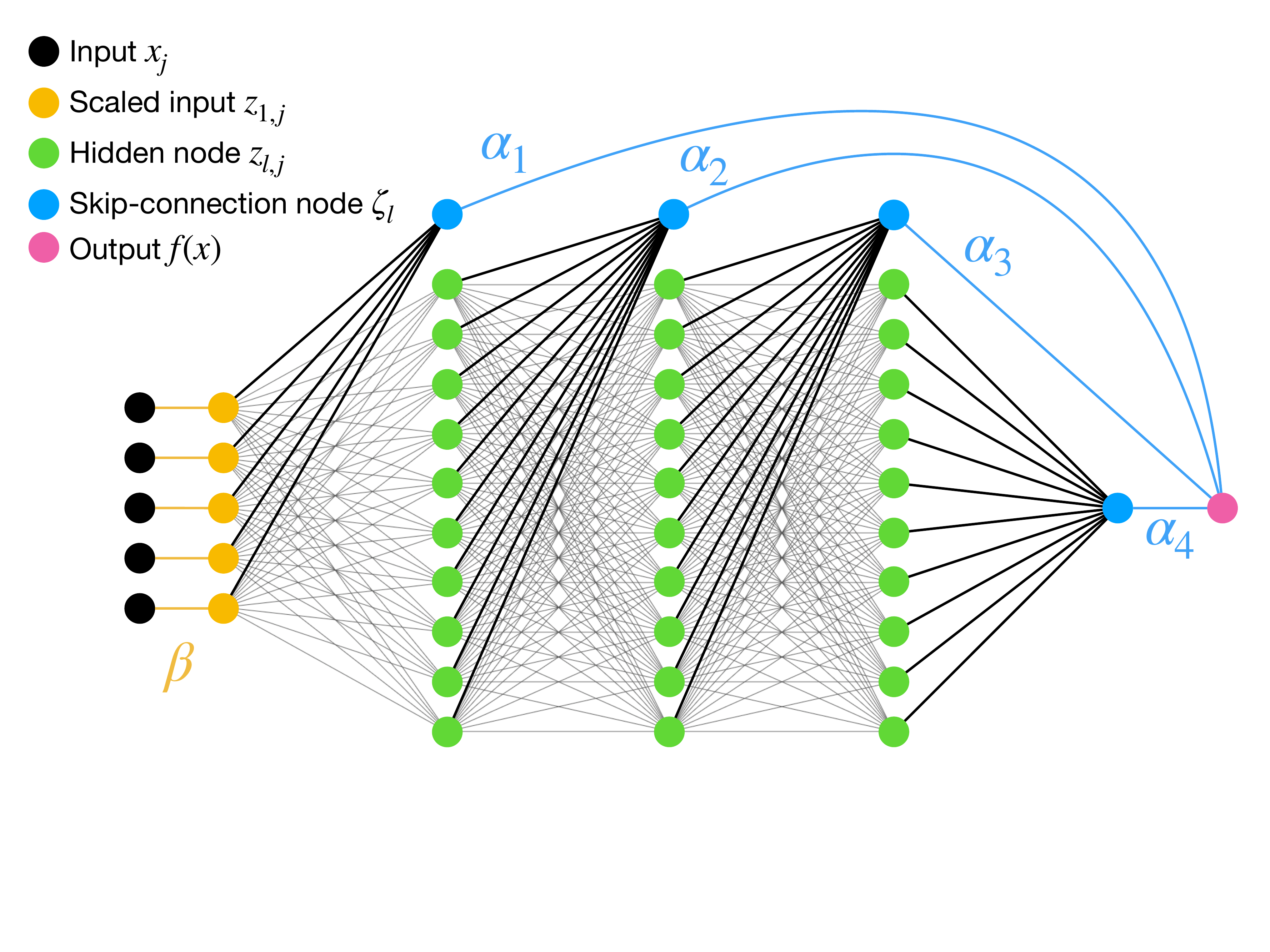}
	\caption{Initial network structure}
	\end{subfigure}
	\begin{subfigure}{0.4\linewidth}
	\includegraphics[width=\linewidth]{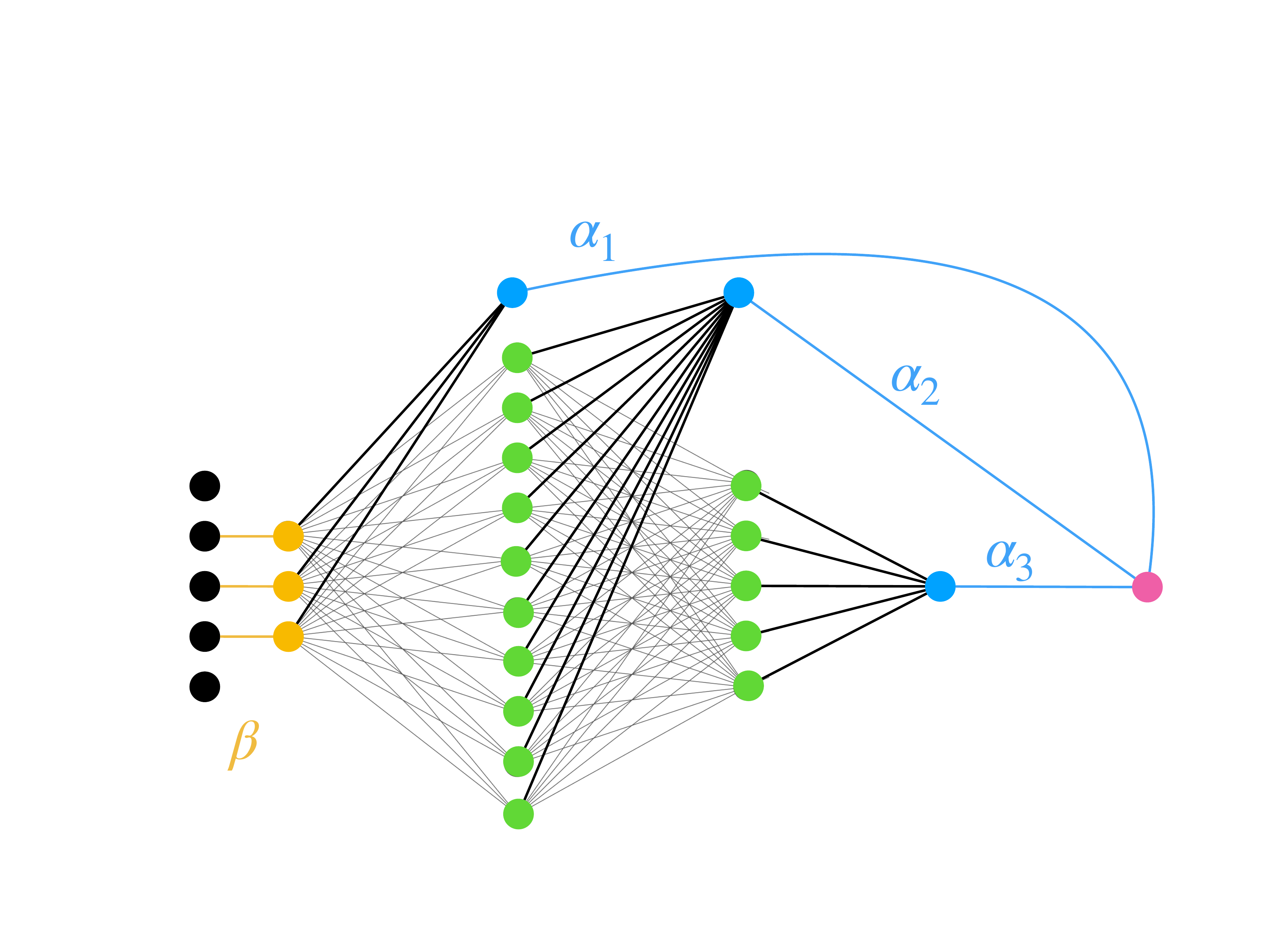}
	\caption{Fitted network}
	\end{subfigure}
	\caption{
	A sparse-input hierarchical network with three hidden layers.
	The network augments a dense neural network by adding i) an input filter layer (yellow edges) that scales the inputs by parameter $\beta$ and ii) skip-connections from each layer to the output.
	The output is the weighted average of the blue nodes with weights $|\alpha_l|$ for $l = 1,..,4$.
	The network is initialized with non-zero weights along all edges.
	During training, many of these edges are set to zero by lasso penalties in the objective, which may prune away inputs, nodes, and entire layers.
	}
	\label{fig:nnet}
\end{figure}

Our method, \textbf{S}parse-\textbf{I}nput hi\textbf{ER}archical networks (SIER-net), modifies a dense neural network by adding an input filter layer and skip-connections from each intermediate layer to the output (Figure~\ref{fig:nnet}).
The structure of the network and the estimation procedure in SIER-net are specially designed to facilitate network structure learning in a data-adaptive manner: the network is first initialized per some maximum allowable network structure and then network nodes and layers are pruned by minimizing an $L_1$-penalized empirical loss.
The network is parameterized by weights $W$ and biases $b$ in the base network, input scaling factors $\beta \in \mathbb{R}^{d}$, skip-connection factors $\alpha \in \mathbb{R}^{L -1}$, as well as additional weights $W'$ and biases $b'$ associated with the skip-connections.
We denote the network as $f_{W, b, \beta, \alpha}$ and provide the full set of equations defining its output in Section~\ref{sec:define} in the Appendix.
For regression problems, training involves minimizing the following objective with loss function $\ell: \mathbb{R} \times \mathbb{R} \mapsto \mathbb{R}$:
\begin{align}
\begin{split}
\min_{W, b, \beta, \alpha} &\  \frac{1}{n}\sum_{i=1}^n \ell \left(f_{W, b, \beta, \alpha}(x_i), y_i \right ) + \lambda_1 \left(
\| \beta \|_1 + \| W'_1 \|_1
\right)  + \lambda_2 \left(
\sum_{l = 1}^{L - 2} \|W_{l}\|_1 + \sum_{l = 2}^{L - 1} \|W'_{l}\|_1
\right)
\end{split}
\label{eq:nn_obj}
\end{align}
where $\lambda_1 \ge 0$ and $\lambda_2 \ge 0$ are penalty parameters.
For classification problems, the $L_1$ penalties are applied to the bias parameters as well to regularize the marginal probability estimates for each class.
Broadly speaking, the first set of $L_1$ penalties scaled by $\lambda_1$ controls the input sparsity and the second set of $L_1$ penalties scaled by $\lambda_2$ controls the number of active layers and hidden nodes.

The input filter layer replaces the first layer with scaled inputs $z_{1, i} = \beta_i x_i$ for $i = 1,...,d$, where $\boldsymbol{\beta}$ is a learnable parameter.
Since the lasso penalty encourages parameters to be exactly equal to zero, minimizing \eqref{eq:nn_obj} will encourage the model to depend on a smaller set of covariates.
So as $\lambda_1$ increases, the size of the fitted model's support will decrease.

To add skip-connections, we add $d_L$ additional nodes per layer, denoted by $\boldsymbol{\zeta}_{l}\in\mathbb{R}^{1 \times d_L}$, to layers $l = 1,\cdots, L - 1$, where $\boldsymbol{\zeta}_{l} = \boldsymbol{z}_{l} W_{l}'  + \boldsymbol{b}_{l}'$ for $W_{l}' \in \mathbb{R}^{d_{l} \times d_L}$ and $\boldsymbol{b}_{l}' \in \mathbb{R}^{1 \times d_L}$.
(Recall that $d_L$ is the number of output nodes from the network, so $d_L = 1$ in the regression setting.)
For each $l = 1,...,L - 1$, nodes $\boldsymbol{\zeta}_l$ are associated with scalar weight $\alpha_{l}$.
We then set the network's output as a function of the weighted average of these additional nodes, i.e.
\begin{align}
f_{W, b, \beta, \alpha}(\boldsymbol{x}) :=
\boldsymbol{z}_{L} = \phi_{\text{out}}\left(
\sum_{l=1}^{L-1} \frac{|\alpha_{l}|}{\sum_{l' = 1}^{L - 1}|\alpha_{l'}|} \boldsymbol{\zeta}_{l}
\right).
\label{eq:sum_connections}
\end{align}
Normalizing the skip-connection weights is useful in a few ways.
First, normalizing the skip-connection weights serves as a form of regularization.
Second, we can modulate the normalized skip-connection weights by varying $\lambda_2$.
In particular, as $\lambda_2$ decreases, the number of active layers and nodes increases and, consequently, the model tends to assign larger skip-connection weights to layers closer to the output.
Finally, the normalized weights can be roughly interpreted as the conditional importance of each layer.
For a formal comparison, we evaluate the proportion of variance contributed by layer $l = 1,....,L - 1$ to the final network output, i.e.
\begin{align}
\Var\left(|\alpha_{l}| \zeta_{l}\right )\big/
\Var\left(\sum_{l'=1}^{L - 1} |\alpha_{l'}| \zeta_{l'}\right ).
\label{eq:var_contrib}
\end{align}

With the aforementioned modifications, the neural network is naturally hierarchical, where simpler models can be represented with fewer nonzero weights.
Consider a few special cases in the regression setting.
If all network parameters are set to zero except for $\boldsymbol{\beta}$, $\alpha_1$, $W'_1$, and $\boldsymbol{b}'_1$, then the fitted network is a linear model.
More generally, if all of the parameters associated with layers $l \ge L'$ are zero for some $L' < L$, then $f_{W, b, \beta, \alpha}$ is a neural network with only $L'$ hidden layers.
Finally, if many entries of $\boldsymbol{\beta}$ are zero, the fitted network is sparse with respect to the number of inputs.
By design, the ordering in the hierarchy corresponds to the axes defined by penalty parameters $\lambda_1$ and $\lambda_2$.

To minimize \eqref{eq:nn_obj} (or find a local minimum), the simplest approach is to run the popular stochastic gradient-based optimization algorithm {\tt Adam} \citep{Kingma2015-oy}.
However, {\tt Adam} does not shrink network weights exactly to zero.
So to obtain a network that is actually sparse, we recommend running a few more steps of (batch) proximal gradient descent \citep{Parikh2014-kd} after {\tt Adam} converges.
Recall that proximal gradient descent is a form of projected gradient descent used to solve non-differentiable convex optimization problems; therefore, it should converge within a few iterations, as \eqref{eq:nn_obj} should be locally quadratic after {\tt Adam} converges.
This two-step optimization procedure is shown in Algorithm~\ref{algo:prox_gd}.
At each iteration of proximal gradient descent, we perform a batch gradient update with respect to the empirical loss, followed by the proximal operator on the two $L_1$ penalties.
The proximal operator, in this case, is very efficient, since it is the soft-thresholding function
\begin{align}
S_{\lambda}(\theta) = \begin{cases}
\theta - \lambda & \text{if } \theta > \lambda \\
0 & \text{if } |\theta| \le \lambda \\
\theta + \lambda & \text{if } \theta < -\lambda
\end{cases}.
\label{eq:soft_thres}
\end{align}

\begin{algorithm}
\begin{algorithmic}
	\STATE{Run Adam until convergence}
	\FOR {$k = 1,2,3...$}
	    \STATE{$\eta^{(k + 1)} := \eta^{(k)} - t_k \nabla_{\eta} \frac{1}{n} \sum_{i=1}^n \ell(f_{\eta}(x_i), y_i) $}
	    \FOR {each parameter $\eta_i$ with an $L_1$ penalty with penalty parameter $\lambda_i$}
	        \STATE{$\eta_i^{(k + 1)} := S_{\lambda_i t_k} (\eta_i^{(k + 1)})$}
	    \ENDFOR
	\ENDFOR
\end{algorithmic}
\caption{SIER-net}
\label{algo:prox_gd}
\end{algorithm}

\subsubsection{Variable screening}
\label{sec:some_theory}
In this section, we derive probability bounds to understand the variable screening properties of SIER-net.
We show that the expected number of false negatives is intimately tied to the predictive value of each variable in the true support, relative to the remaining covariates.
We formalize this idea by defining a predictive value function $\gamma$ in Definition~\ref{def:pred} and then show how it appears in the probability bounds on the number of false negatives.
Our results show that SIER-net is effective at support screening when the predictive value of the true support is high, which is maximized when the covariates are independent.
We then show that no estimation procedures can accurately perform support screening or recovery if the predictive value is low, which occurs when the covariates are highly correlated.
Given this impossibility result, we will introduce an ensembling procedure in Section~\ref{sec:ensembling} that quantifies the uncertainty of the true support instead.
The proofs for results in this section are given in Section~\ref{sec:proofs} of the Appendix.

We start by introducing some notation and definitions.
Let $\mathcal{F}$ be a class of functions that map from $\mathcal{X}$ to $\mathcal{Y}$.
For any function $f \in \mathcal{F}$, define its support $s(f)$ as the set $\tilde{s} \subseteq \{1,...,d\}$ such that $f(x) = f(x')$ for all $x, x' \in \mathcal{X}$ where $x_i = x'_i$ for all $i \in \tilde{s}$ (If there are multiple possible supports, use a random rule to select one among those with the smallest cardinality).
Also, for any set $\tilde{s} \subseteq \{1,...,d\}$, let $\mathcal{F}_{\tilde{s}}$ be the set of functions $f \in \mathcal{F}$ with support $s(f) \subseteq \tilde{s}$ and let $|\tilde{s}|$ denote its cardinality.
Suppose we observe IID observations $(X_i,Y_i)$ for $i = 1,...,n$ drawn from joint distribution $P$.
Let its conditional mean be denoted $\mu_P(x) = \E_P[Y | X= x]$ and, for convenience, let the true support be denoted $s^*_P = s(\mu_P)$.

Let $\Theta_{L, B_0, B_1}$ be the set of all sparse-input hierarchical network parameters with up to $L$ layers, infinity norm up to $B_0$, and total variation norm up to $B_1$.
Here, we use the total variation norm of a neural network as defined in \citet{Barron2019-sb}, which scales with the size of the support and number of hidden nodes.
Given a neural network estimator $f_{\hat{\theta}_n}$, our goal is to characterize the probability of missing at least $m$ relevant covariates, i.e.
\begin{align}
P\left(
\left|
s^*_P \setminus s(f_{\hat{\theta}_n})
\right|
\ge m
\right).
\label{eq:prob_fn}
\end{align}
For simplicity, we suppose the inputs are scaled such that $\mathcal{X} = [-1,1]^d$, consider only the regression setting, and let $\ell$ be the squared error loss.
Moreover, suppose that $\epsilon = Y - \mu(X)$ is a Gaussian random variable with variance $\sigma^2$ and the true regression function is a neural network, i.e. $\mu_P = f_{\theta^*}$ for some $\theta^* \in \Theta_{L, B_0, B_1}$.
The results here can be generalized to sub-Gaussian noise and alternative loss functions.

We quantify the predictive value of individual variables and groups of variables using the following definition.
\begin{definition}
	For set $\tilde{s} \subseteq \{1,...,d\}$, let $f_{\tilde{s}}^*$ be the population risk minimizer with support $\tilde{s}$, i.e.
	\begin{align}
	f_{\tilde{s}}^* = \argmin_{f \in \mathcal{F}_{\tilde{s}}} \E_P \ell(f(X), Y).
	\end{align}
	For integers $m = 1,...,d$, denote the minimum achievable excess risk when omitting $m$ variables from the true support as
	\begin{align}
	\gamma(m; P) = \min_{\tilde{s} \subseteq \{1,...,d\} : |s^*_P \setminus \tilde{s}| = m } \E_P \left\{ \ell(f^*_{\tilde{s}}(X), Y) - \ell(\mu_P(X), Y) \right \}.
	\label{eq:gamma}
	\end{align}
	\label{def:pred}
\end{definition}
\noindent Specifically, a large value for $\gamma(m; P)$ means that any model that fails to include $m$ variables from the true support will have significantly worse prediction accuracy.
Put another way, $\gamma(m; P)$ is large if the variables outside the true support are not highly predictive of the response $Y$.
Next, we show that our ability to perform variable screening is fundamentally tied to $\gamma$.

We derive the following upper bound of \eqref{eq:prob_fn} by combining our definition of $\gamma$ with standard techniques for bounding the excess risk of the constrained empirical risk minimizer.
Although the previous section described fitting the network by minimizing the penalized empirical risk in \eqref{eq:nn_obj}, we analyze the constrained minimizer here since we can directly apply results from \citet{Barron2019-sb}.
Moreover, by using the method of Lagrangian multipliers, the constrained problem can be transformed to a similar structure as the penalized form.
For appropriately chosen penalty parameters, the two estimation procedures can give similar estimates and, in fact, are equivalent in convex settings.
\begin{theorem}
	Given $n$ observations, let $\hat{\theta}_n$ be a global minimizer of the empirical risk
	\begin{align}
	\hat{\theta}_n \in \argmin_{\theta \in \Theta_{L, B_0, B_1}} \frac{1}{n} \sum_{i=1}^n \ell \left (f_{\theta}(x_i), y_i \right ).
	\label{eq:thrm_def}
	\end{align}
	Then for any $m = 1,\ldots,d$ and $\delta \ge 0$, we have
	\begin{align}
	\begin{split}
	P\left(
	\left |s^*_P \setminus s(f_{\hat{\theta}_n}) \right | \ge m
	\right)
    \le
    &\  \mathbbm{1}\left\{
	B_1 (4B_0  + \sigma) L \sqrt{\frac{2(L \log(2) + \log (2d))}{n}} + 2\delta
	\ge \gamma(m; P)
	\right \}\\
	&\ + \exp\left(
	-\frac{n \delta^2}{2 B_0^4}
	\right) + 2\exp\left(
	-\frac{n \delta^2}{8B_0^2 \sigma^2}
	\right)
	.
	\end{split}
	\label{eq:prob_support}
	\end{align}
	\label{thrm:upper}
\end{theorem}
\noindent Using asymptotic notation, the upper bound in \eqref{eq:prob_support} is nontrivial (i.e. smaller than one) only if the number of observations $n \gtrsim 1/\gamma^2(m)$.
Thus, the upper bound is small when the predictive value of the true support relative to the remaining covariates is high.
These conditions are nonparametric generalizations of those used for support screening and variable selection in linear models, i.e. incoherence and beta-min conditions and variants thereof \citep{Buhlmann2011-ko, Wainwright2019-sp}.

In addition, Theorem~\ref{thrm:upper} reveals how variable screening performance and prediction accuracy depend on the model capacity.
In particular, we see that the upper bound grows quickly with respect to $L$, $B_0$, and $B_1$.
So, one should therefore constrain the network structure search to networks with depth, infinity norm, and total variation norm no more than that of the true model.
Indeed, we designed SIER-net based on this theoretical result:
the second set of $L_1$ penalties in \eqref{eq:nn_obj} helps us control network depth;
the sum of all the penalties controls the infinity norm;
and because the first set of penalties controls the input sparsity and the second controls the number of hidden nodes, their sum controls the total variation norm.

Next, we characterize how the best achievable performance by any support screening  procedure depends on $\gamma$.
Consider any non-negative monotonically non-decreasing function $\gamma$.
For any $k \in \{1,...,d\}$, let $\mathcal{P}_{k, \gamma}$ be the set of distributions $P$ with support size up to $k$ where the predictive value of the true support is lower bounded by $\gamma$, i.e.
$$
\mathcal{P}_{k, \gamma} := \left\{
P : |s(\mu_P)| \le k, \gamma(m; P) \ge \gamma(m) \ \forall m = 1,...,k
\right \}.
$$
Thus our goal is to lower bound the minimax probability
\begin{align}
\inf_{\hat{f}_n: |s(\hat{f}_n)| \le k}
\sup_{P\in \mathcal{P}_{k, \gamma}}
P\left(
\left|
s(\mu_P) \setminus s(\hat{f}_n)
\right|
\ge m
\right)
\label{eq:minimax}
\end{align}
for $m = 1, ..., \min(k, \lfloor d/2 \rfloor)$.\footnote{We ignore the case where $m > \lfloor d/2 \rfloor$ since this probability is only relevant if $k > \lfloor d/2 \rfloor$. However, the probability must be zero in this case by the pigeonhole principle.}
For this minimax probability to be meaningful, note that we have removed trivial estimation procedures that select an excessive number of predictors.
Instead, \eqref{eq:minimax} is restricted to estimators that select at most $k$ predictors.
The following lower bound is a straightforward application of Le Cam's method (see \citet{Wainwright2019-sp}).
The proof relates the difficulty of the estimation procedure to the distance between probability models, which in turn can be bounded by $\gamma$.
\begin{theorem}
For any positive integer $m \le \min(k, \lfloor d/2 \rfloor)$, we have that
\begin{align}
\inf_{\hat{f}_n: |s(\hat{f}_n)| \le k}
\sup_{P \in \mathcal{P}_{k, \gamma}}
P\left(
\left|
s(\mu_P) \setminus s(\hat{f}_n)
\right|
\ge m
\right)
\ge
\frac{1}{2}
\left(
1 - \sqrt{\frac{n \sigma^2}{2} \gamma(2m)}
\right).
\end{align}
\label{thrm:lower}
\end{theorem}
\noindent Assuming $\gamma$ is strictly positive, the minimax probability is smaller than $q \in (0, 1/2)$ only if the number of observations $n \gtrsim (1/2 - q)^2/\gamma(2m)$.

From Theorems~\ref{thrm:upper} and \ref{thrm:lower}, we see that the upper and lower bounds on the probability of having $m$ false positives both depend on $\gamma$.
For this probability to be small, the upper bound states that $n \gtrsim 1/\gamma^2(m)$ while the lower bound states that $n \gtrsim 1/\gamma(2m)$.
Admittedly, the bounds differ by a polynomial factor, which we believe could be tightened with additional assumptions about the neural network estimation procedure.
Nevertheless, the results clearly show that our ability to perform support screening  is fundamentally connected to the predictive value of the relevant covariates relative to the irrelevant ones.

Although the proof techniques used in this section are not novel, this connection between support screening and $\gamma$ has important practical implications that, to our knowledge, have not been sufficiently emphasized in the literature.
In particular, we have established that good support recovery is an unrealistic goal if covariates in the true support can be predicted accurately using covariates outside the support.
This issue cannot be ignored since covariates are likely to be correlated in high-dimensional data.
We will address this issue in the following section by constructing an \textit{ensemble} of sparse-input hierarchical networks.

\subsection{Ensemble by averaging sparse-input hierarchical networks}
\label{sec:ensembling}

We now propose ensemble by averaging sparse-input hierarchical networks, which we refer to as EASIER-net.
Ensembling addresses two issues with training a single sparse-input hierarchical network.
First, as shown in Section~\ref{sec:some_theory}, accurate variable screening  is unrealistic when covariates are highly correlated because there is not enough information to distinguish relevant versus irrelevant variables.
Instead, the fitted model should reflect the uncertainty of the true support.
Second, in many high-dimensional datasets, strongly-correlated covariates actually belong to a group, such as gene expression values from the same pathway.
In this case, we would like the variable selection method to select variables from the same group at similar rates.
We outline the ensembling procedure below and engage a Bayesian perspective to help us understand how ensembling addresses these two issues.

In EASIER-net, the ensemble is composed of $B$ independently-trained sparse-input hierarchical networks, denoted $\{\hat{f}_{(1)}, ..., \hat{f}_{(B)} \}$.
We inject noise into the training procedure to reduce correlation across the fitted models and, thereby, the variance of the ensemble \citep{Breiman1996-bf}.
We follow nearly the same recipe as in \citet{Lakshminarayanan2017-wn}: for each network, we randomly initialize the parameters and minimize \eqref{eq:nn_obj} by running {\tt Adam} with shuffled minibatches and then batch proximal gradient descent.
To predict the response for a given input $x$, we average the predictions across networks, i.e. $\hat{f}_{\text{ensemble}} = \frac{1}{B} \sum_{b=1}^B \hat{f}_{(b)}(x)$.
Note that in classification problems, this involves averaging the categorical distributions.
Also, unlike some ensembling procedures, e.g. random forests, we \textit{do not} bootstrap or subsample our data.
In empirical experiments, we found that random network initializations and random mini-batch ordering without data re/sub-sampling achieved stronger performance.

Ensembling can be viewed as an approximation of Bayesian model averaging (BMA), a connection that has been noted in previous works \citep{Pearce2020-ee, Wilson2020-jw}.
BMA places a prior over all possible model classes and averages over the posterior distributions to make a prediction \citep{Madigan1994-ni}.
Ensembling uses a similar approach: since stochastic gradient descent approximates the Bayesian posterior using a single mode \citep{Mandt2017-xk}, the aforementioned ensembling procedure randomly samples maximum a posteriori (MAP) estimates and averages their predictions.
By sampling multiple modes, ensembling often produces better approximations of the posterior distribution than procedures that sample from a single mode \citep{Wilson2020-jw, Pearce2020-ee}.
Since BMA is computationally intractable in many settings, we can instead use ensembling as a computationally efficient approximation.

We use this Bayesian perspective to understand why ensembling can help in settings where covariates are highly correlated.
In variable selection problems, BMA quantifies the uncertainty of the true support via the posterior distribution \citep{Raftery1997-pr, Hoeting1999-pg, Viallefont2001-wa}.
Since EASIER-net is an approximate BMA procedure, the probability that a member in EASIER-net selects support $\tilde{s}$ is approximately equal to the posterior probability that the true support is $\tilde{s}$, i.e. for any $b = 1,..,B$, we have
\begin{align}
\Pr \left (s (\hat{f}_{(b)}) = \tilde{s} ; \{(x_i,y_i): i = 1,\ldots ,n\} \right)
\approx P\left (s(f) = \tilde{s} \mid \{(x_i,y_i): i = 1,\ldots ,n\} \right ).
\label{eq:post_prob}
\end{align}
Thus, the variety of the fitted supports in EASIER-net reflects the uncertainty of the variable selection procedure.
Moreover, \eqref{eq:post_prob} implies that the selection rate of variable $i$ in EASIER-net is approximately the posterior probability of variable $i$ belonging in the true support.
So if all variables in a group are strongly correlated, their selection rates should be similar.

\subsection{Hyperparameter tuning}
We tune the two penalty parameters in EASIER-net using K-fold cross-validation and fix the remaining hyperparameters to their default values (see the Section~\ref{sec:hyperparam} in the Appendix for reasonable defaults).
We recommend performing either a grid search or a random search \citep{Bergstra2012-ct} over candidate penalty parameter values since these procedures are easily parallelizable.
For each candidate value, we apply EASIER-net to each of the $K$ folds and obtain the average validation loss.
We then apply EASIER-net on all of the data using the penalty parameters that minimize the average validation loss.

\section{Empirical analyses}
\label{sec:empirical}

For all analyses, we initialize SIER-net with 5 hidden layers and 100 hidden nodes per layer.
In all cases, EASIER-net was trained with $B = 20$ members, unless specified otherwise.
To handle class imbalance issues in (multi-)classification problems, we weight the observation's loss by the inverse of the empirical frequency of its class.
All continuous covariates and outcomes were centered and scaled to have mean zero and variance one.
Penalty parameters were tuned via 4-fold cross-validation unless specified otherwise.
The mini-batch size for Adam was one-third of the training dataset size.

\subsection{Simulation study: Deconstructing EASIER-net}
\label{sec:empirical_deconstrct}

\begin{table}
\centering
{\small
\begin{tabular}{lrHrrr}
Model & Test loss & val loss & \# layers & Avg \# hidden nodes & \# support \\
\toprule
\multicolumn{5}{c}{\textit{20 relevant, 600 observations}}\\
DropoutNet-Single &   0.497423 &  0.384651 &        5 &            100.0 &         100.0 \\
DropoutNet-Ensemble &   0.355795 &  0.299026 &        3.0 &            100.0 &         100.0 \\
SparseNet-Single &   0.241641 &  0.189753 &        1 &            100.0 &          39.0 \\
SparseNet-Ensemble &   0.226045 &  0.156973 &        3.0 &            100.0 &         100.0 \\
SIER-net &   0.223311 &  0.189781 &        3 &            100.0 &          17.0 \\
EASIER-net &   0.226281 &  0.173028 &        3.0 &            100.0 &          25.6 \\
\midrule
\multicolumn{5}{c}{\textit{20 relevant, 3000 observations}}\\
   DropoutNet-Single &   0.473178 &  0.431147 &        5 &            100.0 &         100.0 \\
DropoutNet-Ensemble &   0.278896 &  0.258332 &        5.0 &            100.0 &         100.0 \\
SparseNet-Single &   0.076208 &  0.068500 &        1 &             65.0 &          25.0 \\
SparseNet-Ensemble &   0.060577 &  0.054455 &        1.0 &            100.0 &          23.6 \\
SIER-net &   0.087473 &  0.078205 &        5 &            100.0 &          43.0 \\
EASIER-net &   0.069780 &  0.060710 &        5.0 &            100.0 &          43.2 \\
\midrule
\multicolumn{5}{c}{\textit{100 relevant, 600 observations}}\\
   DropoutNet-Single &   0.511054 &  0.470498 &        3 &            100.0 &         100.0 \\
DropoutNet-Ensemble &   0.404107 &  0.375981 &        3.0 &            100.0 &         100.0 \\
SparseNet-Single &   0.382923 &  0.362970 &        1 &             71.0 &          58.0 \\
SparseNet-Ensemble &   0.399566 &  0.338478 &        3.0 &            100.0 &         100.0 \\
SIER-net &   0.382481 &  0.390588 &        3 &            100.0 &          64.0 \\
EASIER-net &   0.385061 &  0.345800 &        5.0 &            100.0 &          92.0 \\
\midrule
\multicolumn{5}{c}{\textit{100 relevant, 3000 observations}}\\
   DropoutNet-Single &   0.513196 &  0.481725 &        3 &            100.0 &         100.0 \\
DropoutNet-Ensemble &   0.304877 &  0.275531 &        3.0 &            100.0 &         100.0 \\
SparseNet-Single &   0.271722 &  0.267776 &        1 &             16.5 &         100.0 \\
SparseNet-Ensemble &   0.269778 &  0.267736 &        1.0 &             17.7 &         100.0 \\
SIER-net &   0.264340 &  0.264534 &        2 &             32.0 &         100.0 \\
EASIER-net &   0.256367 &  0.262547 &        2.0 &             86.4 &         100.0 \\
\end{tabular}
}
\caption{
We deconstruct EASIER-net to understand how each of the three modifications --- sparsity, ensembling, and skip-connections --- affect the test loss and the selected network structure.
The baseline method \texttt{DropoutNet} only regularizes the model using dropout and requires tuning one hyperparameter.
Next, the \texttt{SparseNet} method adds lasso penalties and increases the number of hyperparameters to three.
Then, \texttt{SIER-net} adds skip-connections and reduces the number of hyperparameters to two; its ensemble version is \texttt{EASIER-net}.
For \texttt{DropoutNet} and \texttt{SparseNet}, we denote the two single and ensemble versions using \texttt{-Single} and \texttt{-Ensemble}.
}
\label{table:simulation_deconstruct}
\end{table}

The purpose of this simulation is to study the contribution from each of the proposed modifications in EASIER-net.
We do this by deconstructing EASIER-net, starting from a dense neural network and adding modifications progressively.

We consider three model-fitting procedures and ensemble versions of each one.
The baseline model (\texttt{DropoutNet}) is a typical dense neural network that is regularized using a dropout rate of 15\%; its ensemble version is similar to that in \citet{Olson2018-vs}.
For \texttt{DropoutNet} and its ensemble version, we only tune the number of hidden layers.
Next, \texttt{SparseNet} adds an input filter layer and regularizes the network weights using two $L_1$ penalties, one to encourage sparsity in the inputs and another for the network weights.
For \texttt{SparseNet} and its ensemble version, we tune the number of hidden layers and two $L_1$ penalty parameters.
Finally, \texttt{SIER-net} adds skip-connections to \texttt{SparseNet}, keeps the initial number of hidden layers fixed, and reduces the number of hyperparameters to two.
All networks were initialized with 100 hidden nodes per layer.

We consider four different simulation setups to reflect a variety of probable settings.
Covariates $(X_1,...,X_{100})$ are sampled independently from $\text{Unif}(0,1)$.
Response $Y$ is defined as
$$
Y = \sum_{i=0}^{m} \left( \sin(2x_{4i + 1} + 2x_{4i + 2}) + 5 x_{4i + 3} \left |x_{4i + 4} - 0.25\right | \right) + \epsilon
$$
where $\epsilon$ is a Gaussian random variable with variance chosen so that the signal to noise ratio is 2.
To vary the sparsity of the true data-generating mechanism, we consider $m = 4$ and $19$, which correspond to 20 and 100 relevant covariates, respectively.
We also vary the amount of training data by considering 600 versus 3000 observations.
We tune the hyperparameters using a single training validation split, where the number of validation samples is one-fourth of that for training.

As shown in Table~\ref{table:simulation_deconstruct}, sparsity and ensembling improve the test loss in all four simulation settings.
Sparsity offers the most improvement when the true model truly depends on a small subset of the covariates, but is still helpful when the true model depends on all the covariates.
Ensembling offers the most improvement to the dense neural network because sparse models tend to be smaller and have more similar structures.
The improvement in prediction accuracy by ensembling sparse-input hierarchical networks appears to depend on the data.
Although ensembling seemed to improve accuracy only slightly in these simulations, we find that it provides significant gains on certain datasets in Section~\ref{sec:uci}.

Ensemble sparse neural networks and EASIER-net have very similar predictive accuracy in all simulation settings, which is exactly as desired.
Recall that the former requires tuning three hyperparameters whereas the latter requires tuning only two.
Thus, the skip-connections in EASIER-net help reduce the amount of computation time spent on network structure learning without sacrificing predictive accuracy.

Finally, we give an example to illustrate how varying penalty parameter $\lambda_2$ modulates the dependence on each of the network layers.
Figure~\ref{fig:continuous_connection_factors} plots the proportion of variance contributed by each layer, as defined in \refeq{eq:var_contrib}.
For large values of $\lambda_2$, the proportion of variance contributed by the scaled inputs (layer $l = 1$) is one, which means the neural network is simply a linear model.
As $\lambda_2$ decreases, the contribution from scaled inputs decreases, and the contribution by the hidden layers increase in the order of their depth; the first and second hidden layers begin contributing earliest, then the third and fourth, and finally the fifth.
As expected, the fifth hidden layer contributes the most as $\lambda_2$ approaches zero, since the last hidden layer has the highest model capacity.

\begin{figure}
	\centering
	\includegraphics[width=0.35\linewidth]{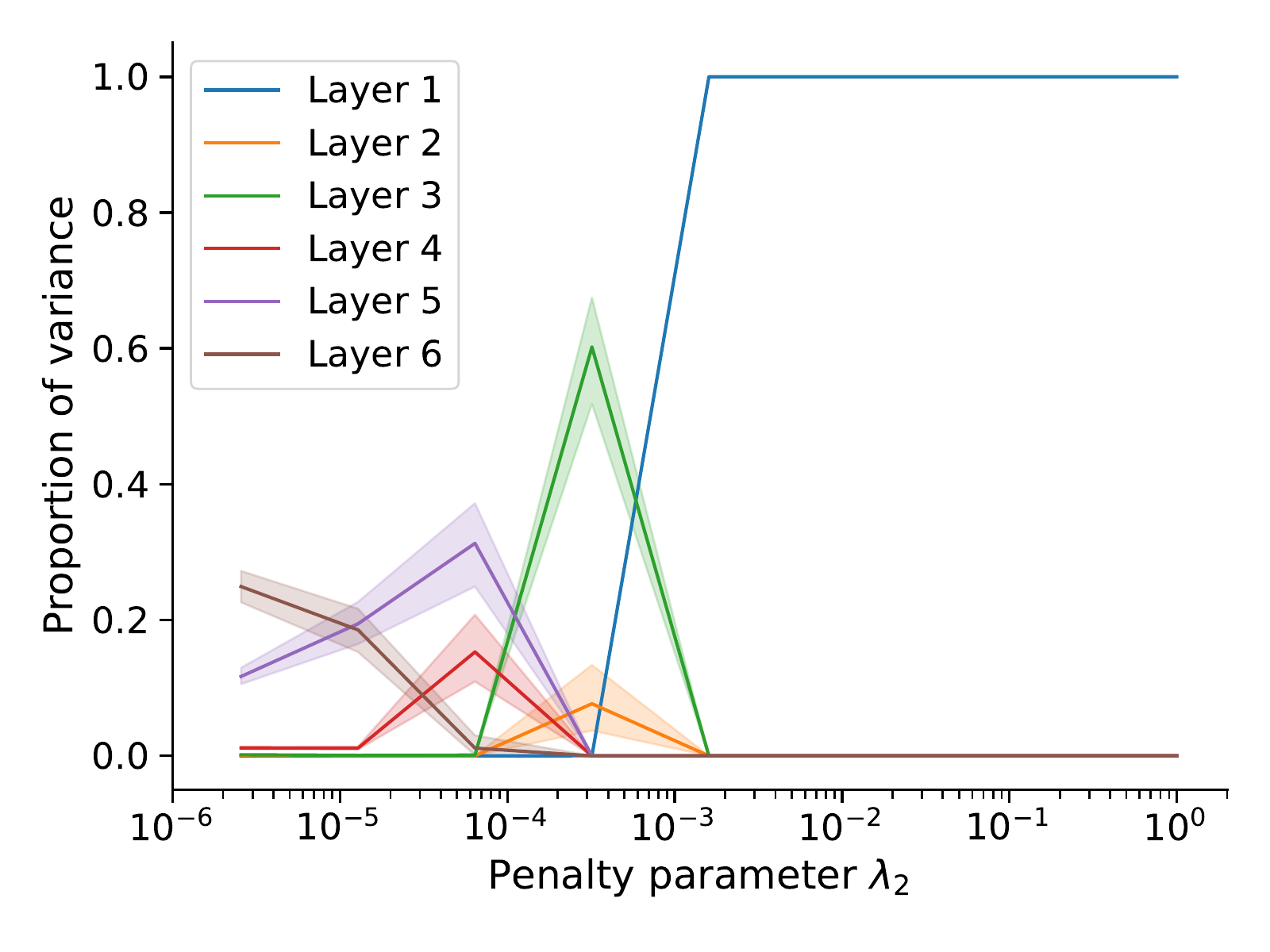}
	\caption{
		Illustration of how varying the penalty parameter $\lambda_2$ controls the proportion of variance contributed by each layer to the final prediction, which thereby controls the size of the fitted network.
		When $\lambda_2 = 1$, the fitted network only depends on the first layer, which means it is linear in the inputs.
		As $\lambda_2$ decreases, the higher levels contribute more to the final prediction and the lower levels contribute less.
	}
\label{fig:continuous_connection_factors}
\end{figure}

\subsection{Simulation study: Variable selection}
\label{sec:empirical_support}

In this section, we study the two claims in this paper regarding variable selection by EASIER-net: i) the number of false negatives should be small if the truly relevant covariates contain predictive information that the other covariates do not and ii) highly correlated variables should be selected with similar proportions across the ensemble.

We run a simulation study where we vary the correlation between the covariates.
For different values of $\rho \in [0,1]$, we generate eight covariates $\boldsymbol{X}$ by sampling $\tilde{X}_i$ from $\text{Unif}(0,1)$ for $i = 1,...,8$ and defining $X_{i}  = \tilde{X}_i$ and $\quad X_{i + 4} = \rho \tilde{X}_i + ( 1 - \rho) \tilde{X}_{i + 4}$ for $i = 1,...,4$.
The response only depends on the first four through the equation $Y = X_1 * X_2 + \sin(X_3 + X_4) + \epsilon$, where $\epsilon$ is a Gaussian random variable with variance such that the signal to noise ratio is 2.
We generated a total of 500 observations.
To get accurate estimates of the variable selection rates, we fit EASIER-net with $B = 100$ members rather than 20.

In this simulation, EASIER-net was highly effective at support screening for $\rho$ from 0 to 0.8 (Table~\ref{eq:support_prob}).
In fact, we had perfect support recovery in most cases, which is a more difficult task than support screening.
For $\rho = 0.8$, only one member in the ensemble omitted the first covariate and selected the fifth one instead.
In addition, these results show that most member networks will select the same support in settings where support screening is effective using a single network.
This means that ensembling does not make sparse-input hierarchical networks less interpretable when support screening is feasible.

The simulation results also show a clear grouping effect from ensembling.
As $\rho$ increases, the selection probabilities of the correlated covariates $X_i$ and $X_{i + 4}$ get closer, showing evidence of the grouping effect.
In the extreme case where covariates $(X_1,...,X_4)$ are identical to $(X_5,...,X_8)$ (i.e. $\rho = 1$), the selection probabilities for $X_i$ and $X_{i + 4}$ are similar for all $i = 1,...,4$.
This is unsurprising since no procedure can differentiate which covariates are truly relevant in this setting.

\begin{table}
\centering
{\footnotesize
\begin{tabular}{c|cccc|cccc}
	\toprule
	{} & \multicolumn{8}{c}{Covariate index} \\
	$\rho$ &               1 &     2 &     3 &     4 &     5 &     6 &     7 & 8\\
	\midrule
	0.00 &        1.000 & 1.000 & 1.000 & 1.000 & 0.000 & 0.000 & 0.000 & 0.000 \\
	0.50 &        1.000 & 1.000 & 1.000 & 1.000 & 0.000 & 0.000 & 0.000 & 0.000 \\
	0.80 &        0.980 & 1.000 & 1.000 & 1.000 & 0.020 & 0.000 & 0.000 & 0.000 \\
	0.90 &        0.860 & 0.890 & 0.870 & 0.990 & 0.160 & 0.120 & 0.280 & 0.010 \\
	0.95 &        0.630 & 0.730 & 0.540 & 0.940 & 0.380 & 0.300 & 0.570 & 0.270 \\
	1.00 &        0.550 & 0.480 & 0.470 & 0.570 & 0.460 & 0.540 & 0.600 & 0.640 \\
	\bottomrule
\end{tabular}
}
\caption{
The proportion of networks in EASIER-net that select each of the covariates.
Only the first four covariates are in the true support.
We simulate covariates such that the correlation between $X_i$ and $X_{i + 4}$ is $\rho$ for $i = 1,...,4$.
For small $\rho$, the model recovers the true support; for big $\rho$, the variable selection rates reflect the uncertainty of the true support.
}
\label{eq:support_prob}
\end{table}

\subsection{Evaluation on UCI datasets}
\label{sec:uci}
We study the performance of EASIER-net on five classification problems and six regression problems from the UCI repository \citep{Dua:2019}, whose details are shown in Table~\ref{table:dataset_descrips} of the Appendix.
To represent a variety of problem settings, the selected datasets vary in sample size, number of features, and number of classes.
We compare against three other classifiers: logistic or linear regression with the lasso, random forests, and gradient boosted trees.
Based on \citet{Olson2018-vs}, we also fit an ensemble of dense neural networks with dropout with ten hidden layers and 100 hidden nodes per layer.
We used the negative log likelihood and squared error as the loss functions for the classification and regression tasks, respectively.

EASIER-net attains the smallest test loss across the regression and classification datasets on average.
In particular, they achieve the best performance on two of the five classification problems and two of the six regression problems (Figure~\ref{fig:uci_res}).
We found that EASIER-net typically has similar or better performance than SIER-net, presumably because they are able to quantify model uncertainty.
Among the remaining datasets, SIER-net, Lasso with linear or logistic regression, random forests, and XGBoost attained top performance in at least one dataset.
On the other hand, deep neural networks with dropout never achieved top performance on any dataset, because they can almost always benefit from network structure pruning.

The computation time for EASIER-net was on the same order as the other methods.
For most of the datasets, the time to fit a single sparse-input hierarchical network took only a few minutes; for the largest dataset CT slices with 53500 observations, this took around forty minutes.
By parallelizing both the ensembling and cross-validation procedures, the total time for running EASIER-net should take around double the time to fit a single sparse-input hierarchical network.
Thus, given the prediction accuracy and computation time of EASIER-net, we conclude that they are a useful addition to the data analyst's toolbox.

Across the eleven datasets, EASIER-net selected a variety of network structures via cross-validation (Table~\ref{table:model_structs_uci}).
In three cases, the selected model was linear in the input.
In the remaining datasets, the network used at least one hidden layer, and the layer that contributed the most variance varied.
We also show the variable selection rates for each dataset in Figure~\ref{fig:uci_support}.
The agreement across member networks in the ensemble depends on the dataset.
For instance, the agreement was high on the soybean dataset and was much lower on the arrhythmia, semieon, and Iran house datasets.
This is consistent with our analysis: the soybean dataset contains a small number of high-level features that likely measure different pieces of information whereas many features in the arrhythmia dataset were extracted from the same channel \citep{Guvenir1997-uk}.

%

\begin{figure}
\centering
\hspace{-1.2in}
\begin{subfigure}{0.4\linewidth}
\includegraphics[width=\textwidth]{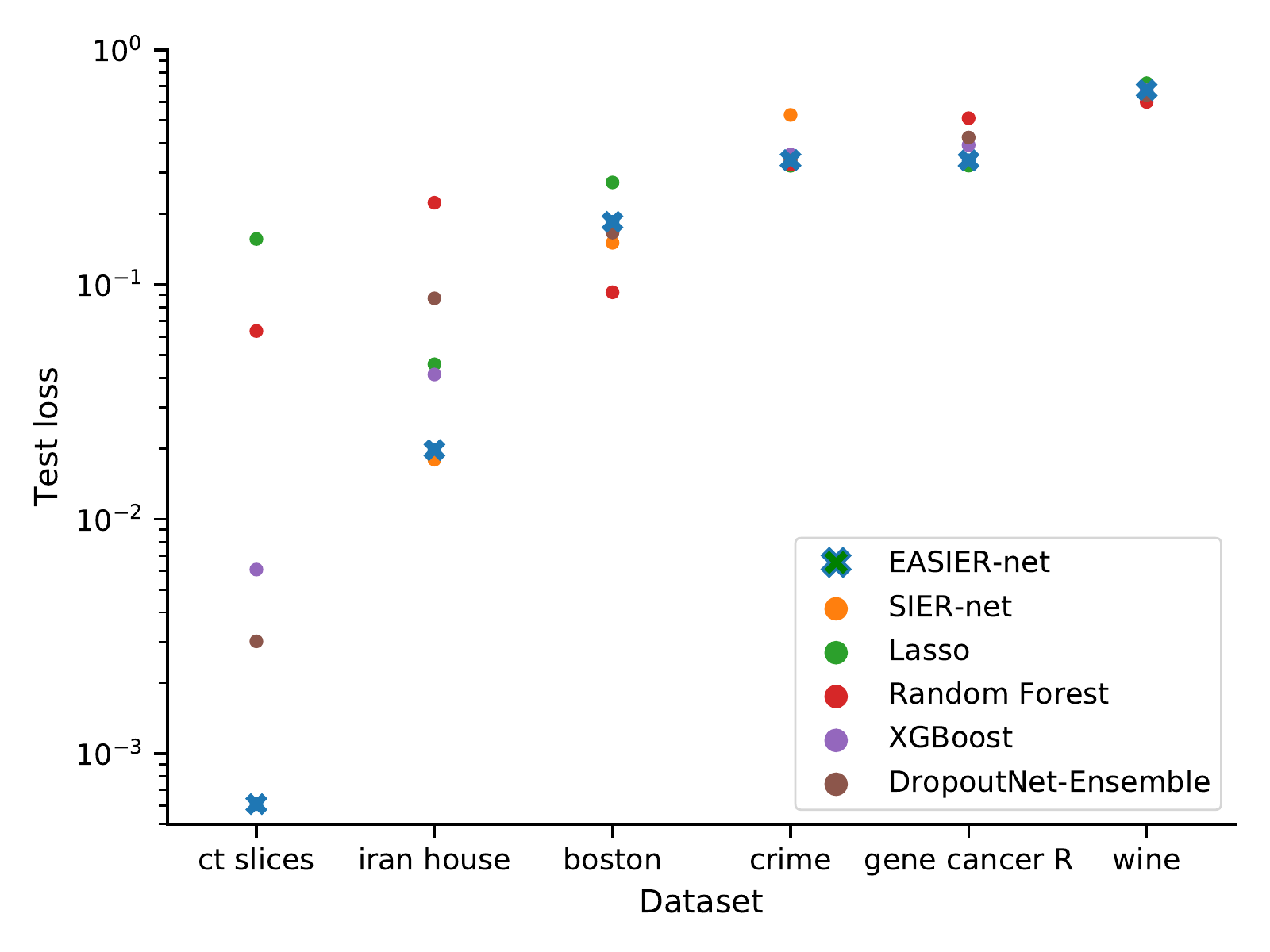}
\end{subfigure}
\hspace{-0.2in}
\begin{subfigure}{0.4\linewidth}
\includegraphics[width=\textwidth]{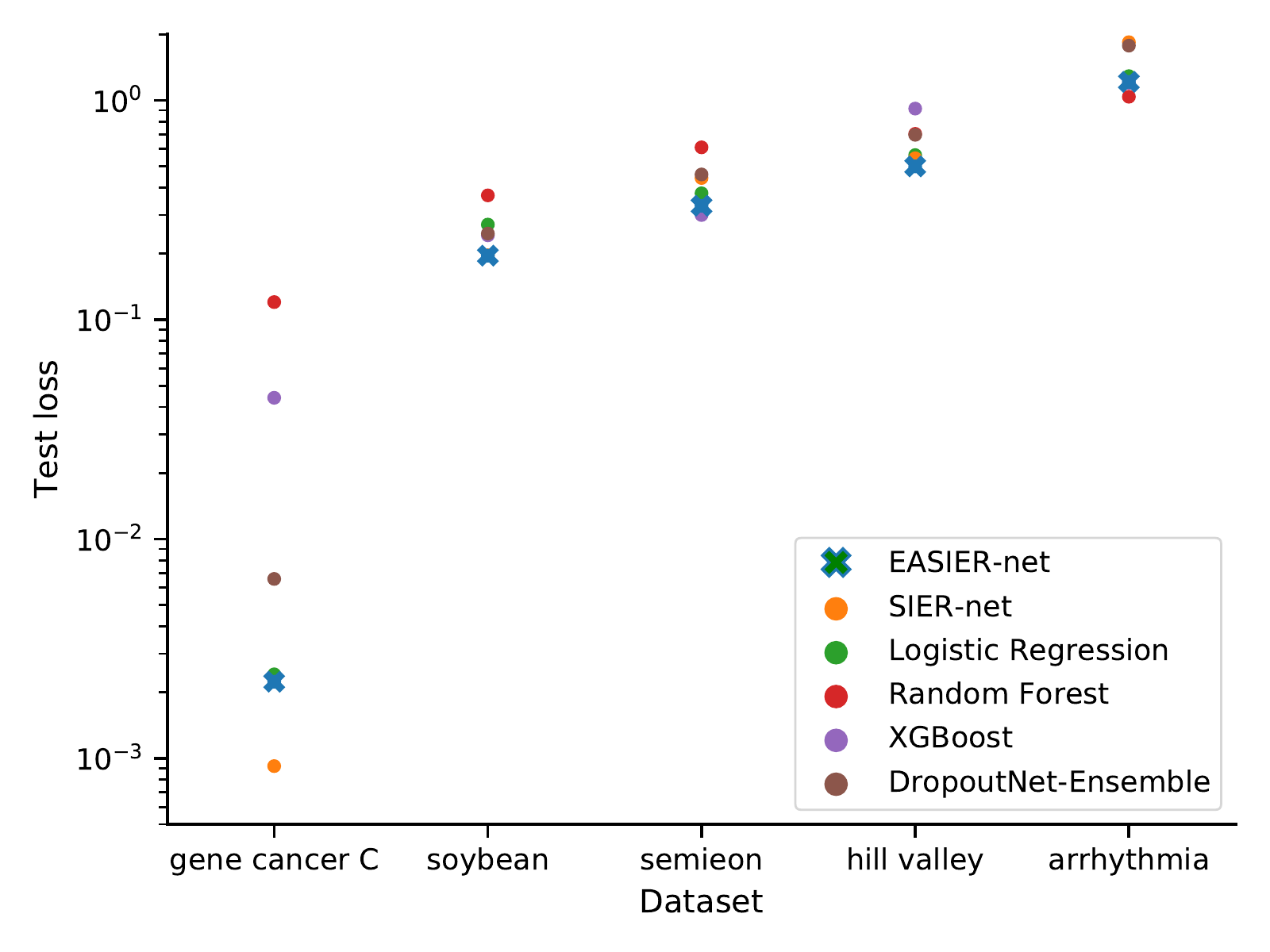}
\end{subfigure}
\begin{subfigure}{0.2\linewidth}
{\footnotesize
\begin{tabular}{p{0.78in}cc}
&\multicolumn{2}{c}{Average test loss} \\
Method & Regression  & Classification \\
\toprule
EASIER-net          & \textbf{0.259}   &   \textbf{0.448}\\
SIER-net & 0.286 &  0.605 \\
Lasso     & 0.306   &  0.500\\
RF  & 0.302 &  0.568\\
XGBoost          & 0.273&    0.541\\
DropoutNet              & 0.276 &  0.637
\end{tabular}
}
\end{subfigure}
\caption{
	Comparison of test loss of different estimation methods across regression (left) and classification (right) datasets.
	The average test loss for each method is shown in the table.
}
\label{fig:uci_res}
\end{figure}

\begin{table}
	\centering
	{\footnotesize
	\begin{tabular}{lHHp{1.9cm}|rrrrrr}
		\toprule
		&&&& \multicolumn{6}{c}{Proportion of variance contributed by layer} \\
		Dataset &   layers &  Support size &  Avg \# hidden nodes &  1 &  2 &  3 &  4 &  5 & 6 \\
		\midrule
		\multicolumn{10}{c}{\textit{Classification tasks}}\\
  gene cancer C &  0.00 (0.00) &  20013.70 (5.32) &     0.0 (0.0) &  \textbf{1.00} (0.00) &  0.00 (0.00) &  0.00 (0.00) &  0.00 (0.00) &  0.00 (0.00) &  0.00 (0.00) \\
            soybean &  0.00 (0.00) &     22.35 (0.17) &     0.0 (0.0) &  \textbf{1.00} (0.00) &  0.00 (0.00) &  0.00 (0.00) &  0.00 (0.00) &  0.00 (0.00) &  0.00 (0.00) \\
            semieon &  5.00 (0.00) &     86.50 (1.68) &   100.0 (0.0) &  0.00 (0.00) &  0.00 (0.00) &  0.00 (0.00) &  0.00 (0.00) &  0.00 (0.00) &  \textbf{0.99} (0.00) \\
        hill valley &  5.00 (0.00) &    100.00 (0.00) &   100.0 (0.0) &  0.00 (0.00) &  0.00 (0.00) &  0.01 (0.00) &  0.01 (0.00) &  0.02 (0.01) &  \textbf{0.61} (0.03) \\
         arrhythmia &  2.80 (0.17) &     34.45 (0.60) &    60.0 (6.1) &  0.00 (0.00) &  0.00 (0.00) &  \textbf{0.61} (0.10) &  0.32 (0.09) &  0.07 (0.03) &  0.00 (0.00) \\
		\midrule
		\multicolumn{10}{c}{\textit{Regression tasks}}\\
        crime &  2.05 (0.05) &    8.05 (0.11) &     6.5 (0.6) &  0.00 (0.00) &  0.00 (0.00) &  \textbf{0.95} (0.05) &  0.05 (0.05) &  0.00 (0.00) &  0.00 (0.00) \\
    CT slices &  5.00 (0.00) &  384.00 (0.00) &   100.0 (0.0) &  0.02 (0.00) &  0.00 (0.00) &  0.00 (0.00) &  0.00 (0.00) &  0.03 (0.01) &  \textbf{0.46} (0.02) \\
       boston &  5.00 (0.00) &   11.60 (0.18) &   100.0 (0.0) &  0.00 (0.00) &  0.00 (0.00) &  0.00 (0.00) &  0.00 (0.00) &  0.10 (0.02) &  \textbf{0.53} (0.04) \\
   Iran house &  5.00 (0.00) &   12.25 (0.29) &   100.0 (0.0) &  0.00 (0.00) &  0.00 (0.00) &  0.00 (0.00) &  0.01 (0.00) &  0.08 (0.00) &  \textbf{0.32} (0.01) \\
         wine &  1.30 (0.11) &   11.00 (0.00) &     9.9 (4.9) &  0.01 (0.01) &  \textbf{0.71} (0.11) &  0.30 (0.11) &  0.00 (0.00) &  0.00 (0.00) &  0.00 (0.00) \\
  gene cancer &  0.00 (0.00) &   15.90 (0.26) &     0.0 (0.0) &  \textbf{1.00} (0.00) &  0.00 (0.00) &  0.00 (0.00) &  0.00 (0.00) &  0.00 (0.00) &  0.00 (0.00) \\
		\bottomrule
	\end{tabular}
	}
\caption{
Summaries of the structure of member networks from EASIER-net: the average number of hidden nodes per active layer and the proportion of variance contributed by each layer.
Standard errors are shown in parentheses.
We highlight the layer that contributes the most to the network prediction.
}
	\label{table:model_structs_uci}
\end{table}

\begin{figure}
\begin{subfigure}{0.49\textwidth}
\includegraphics[width=0.32\linewidth]{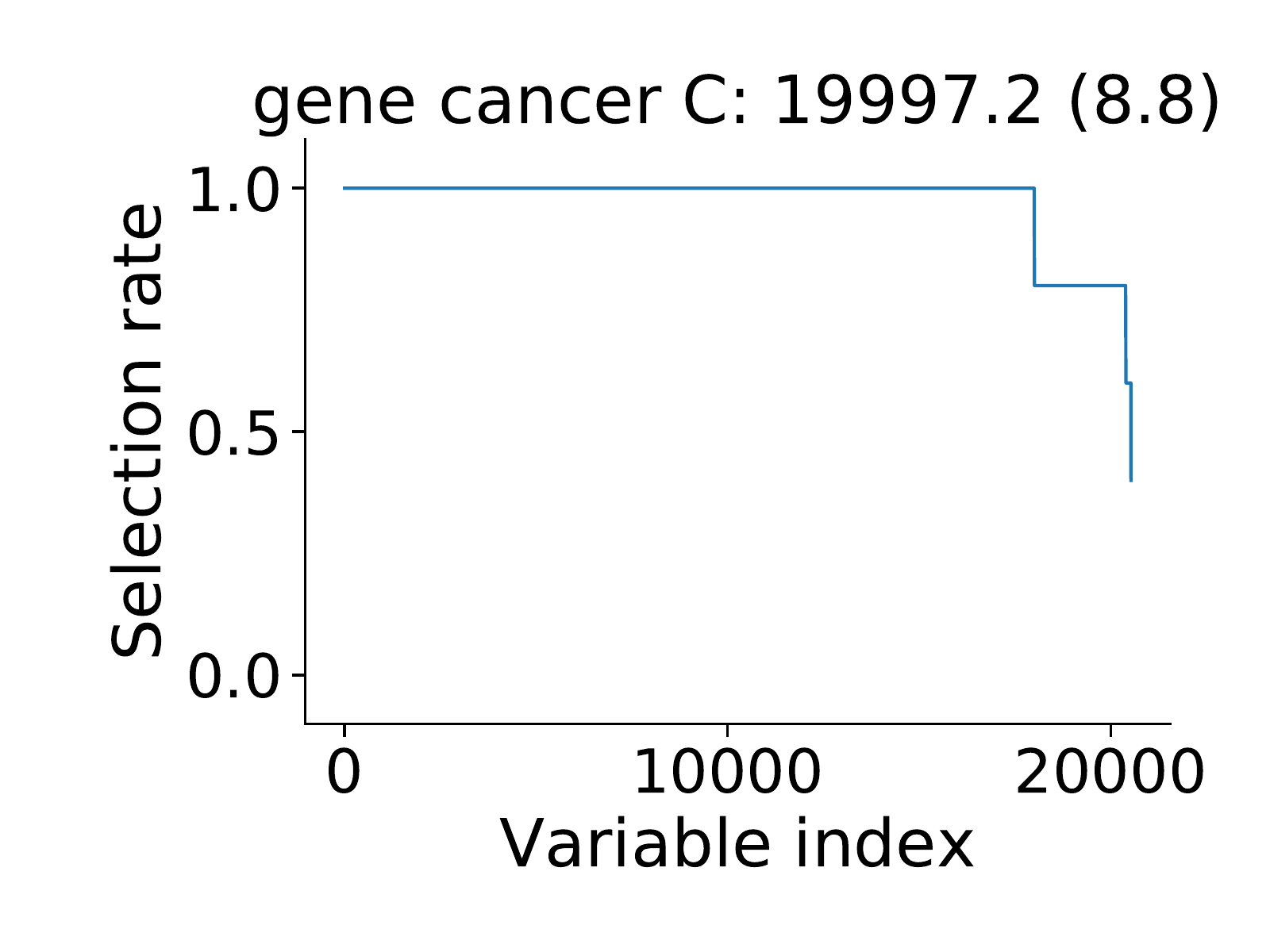}
\includegraphics[width=0.32\linewidth]{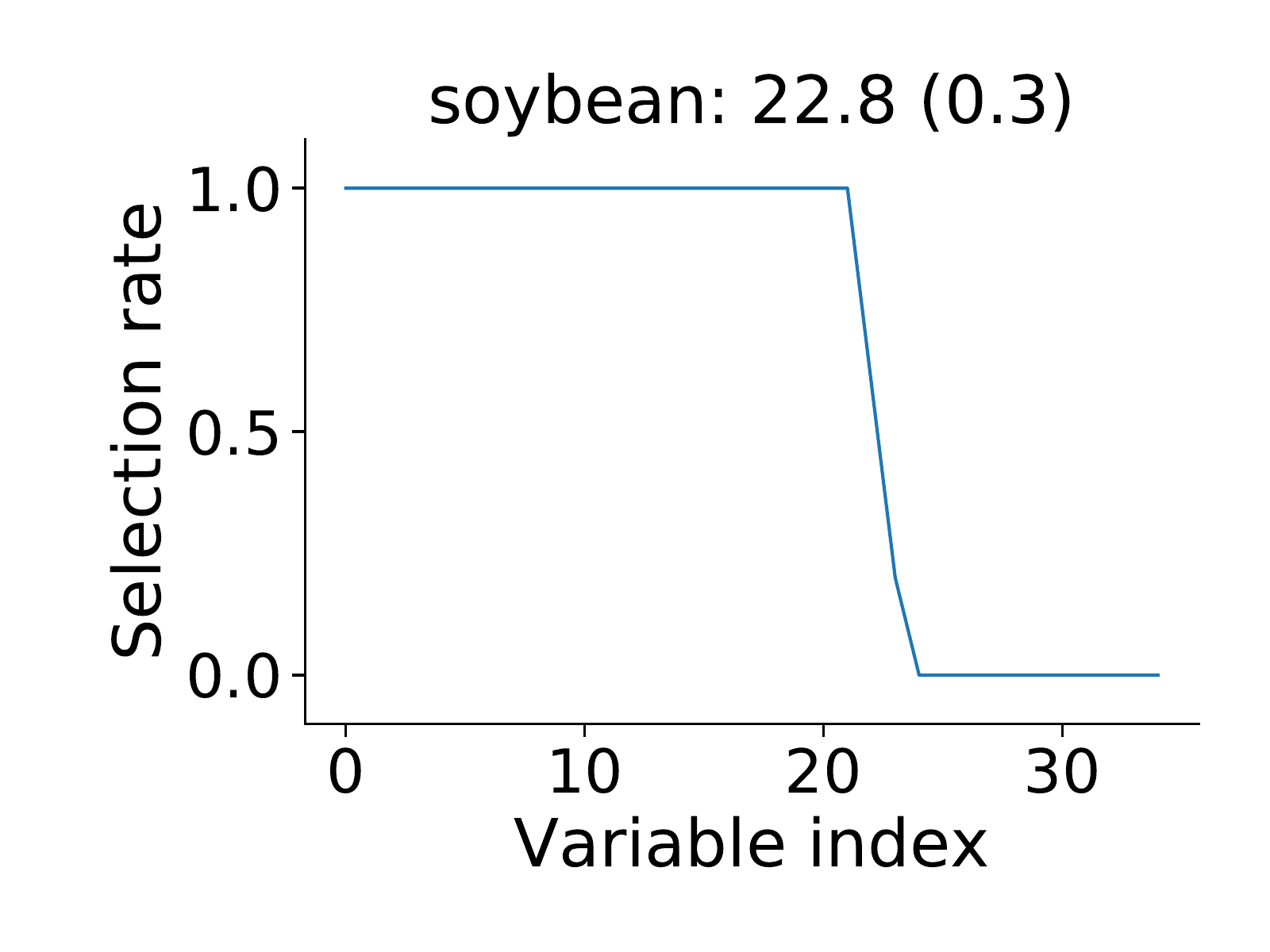}
\includegraphics[width=0.32\linewidth]{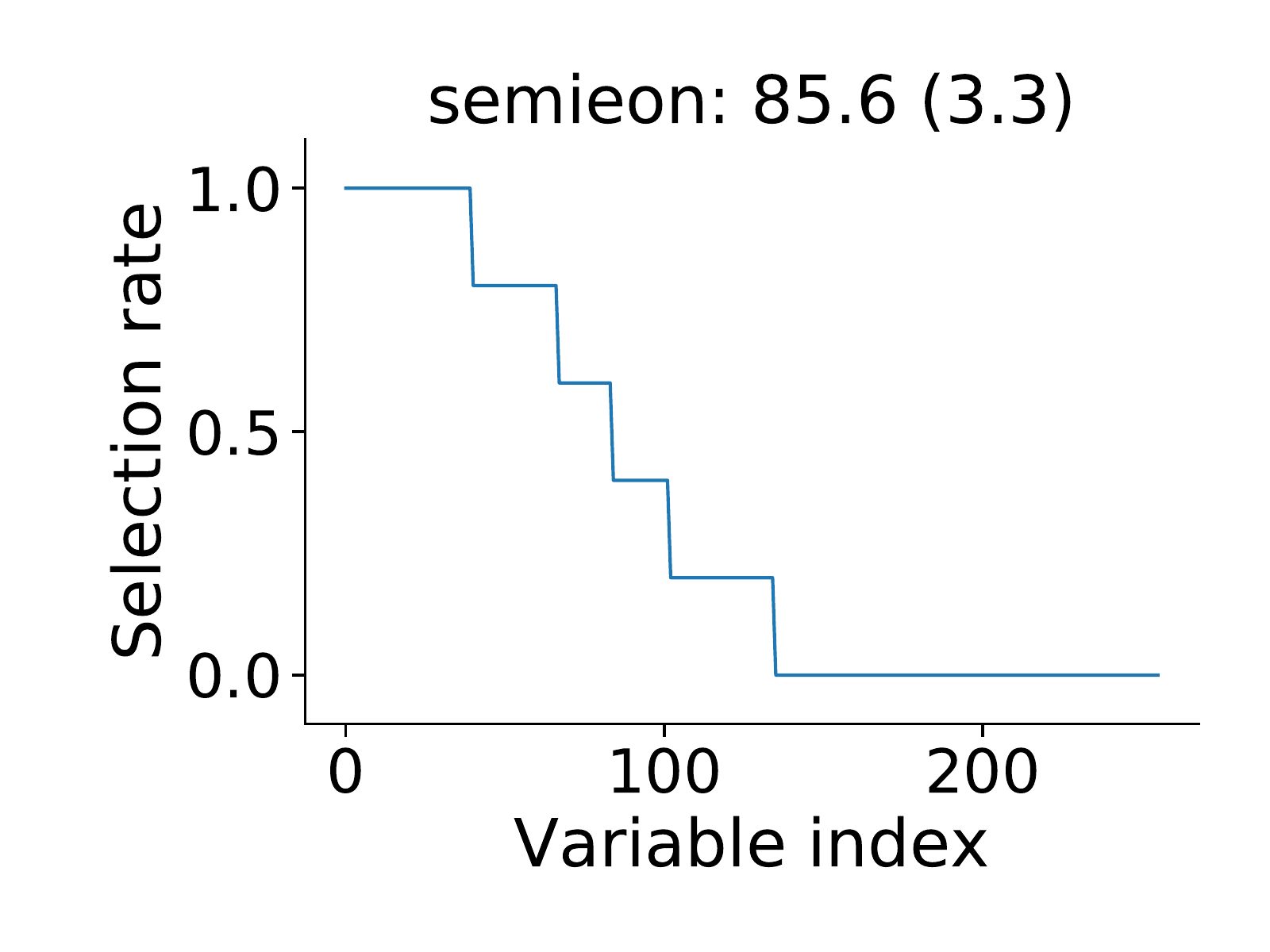}
\includegraphics[width=0.32\linewidth]{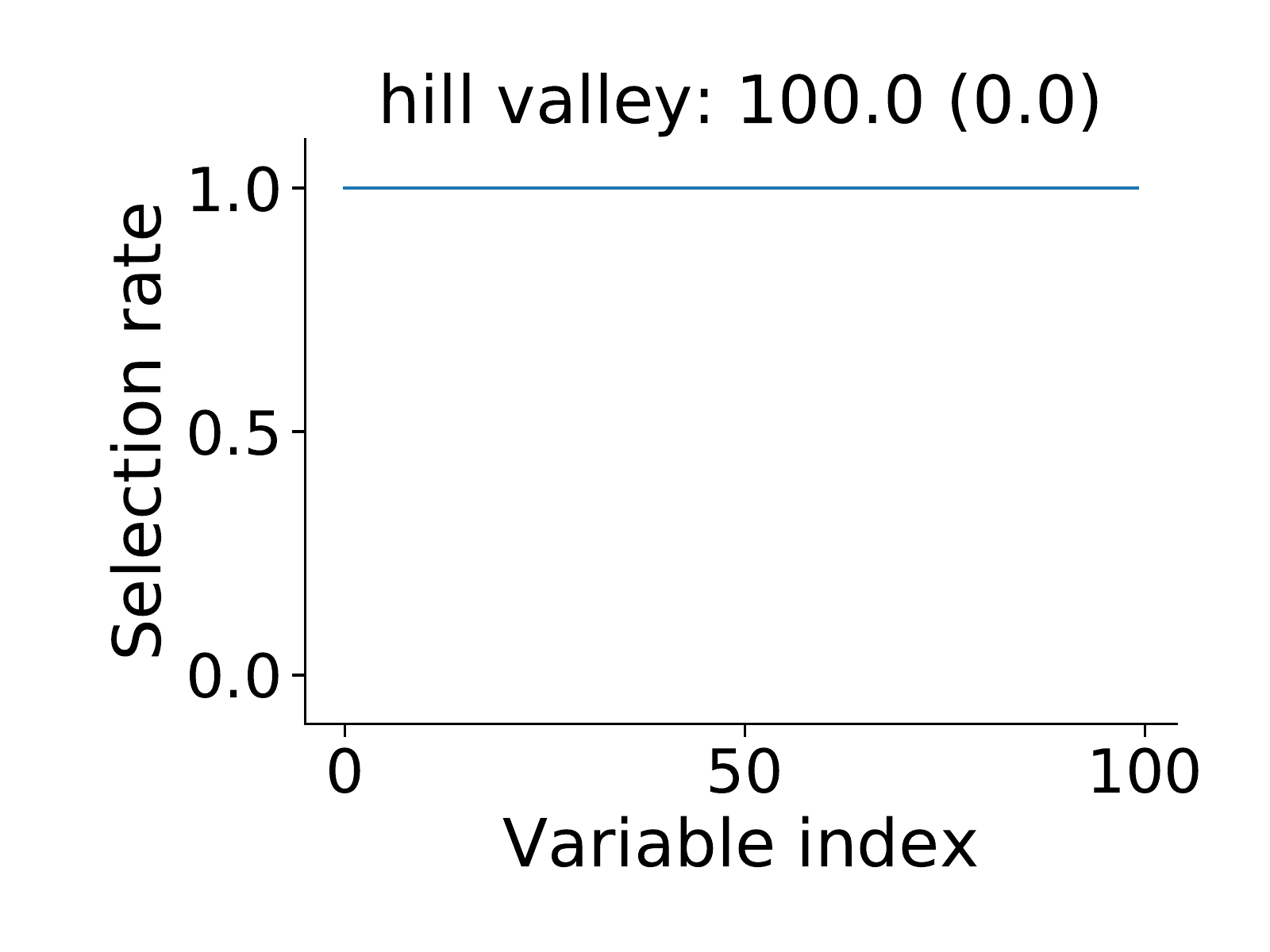}
\includegraphics[width=0.32\linewidth]{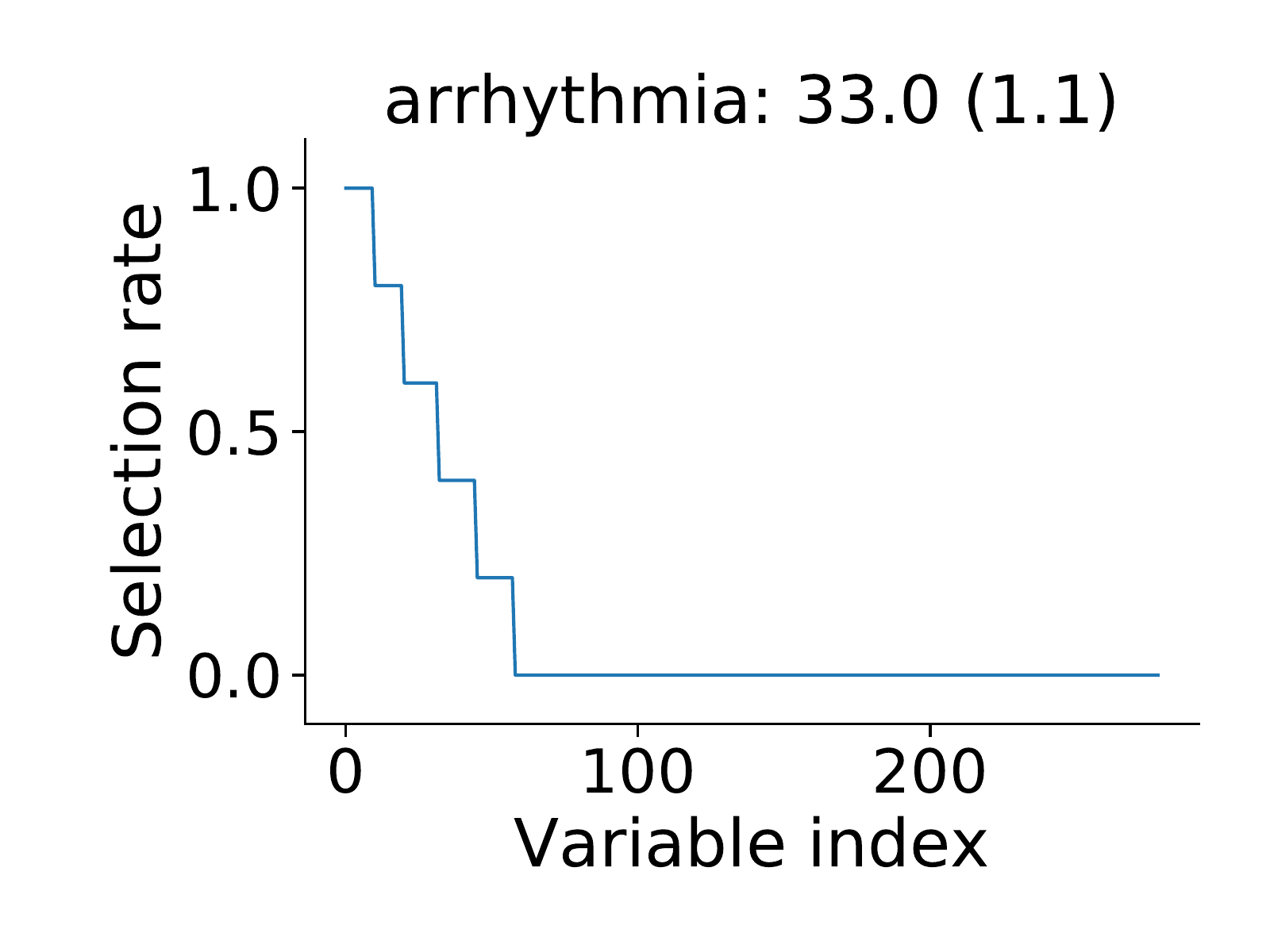}
\caption{Classification}
\label{fig:class_support}
\end{subfigure}
\hspace{0.001\textwidth}
\begin{subfigure}{0.49\textwidth}
\includegraphics[width=0.32\linewidth]{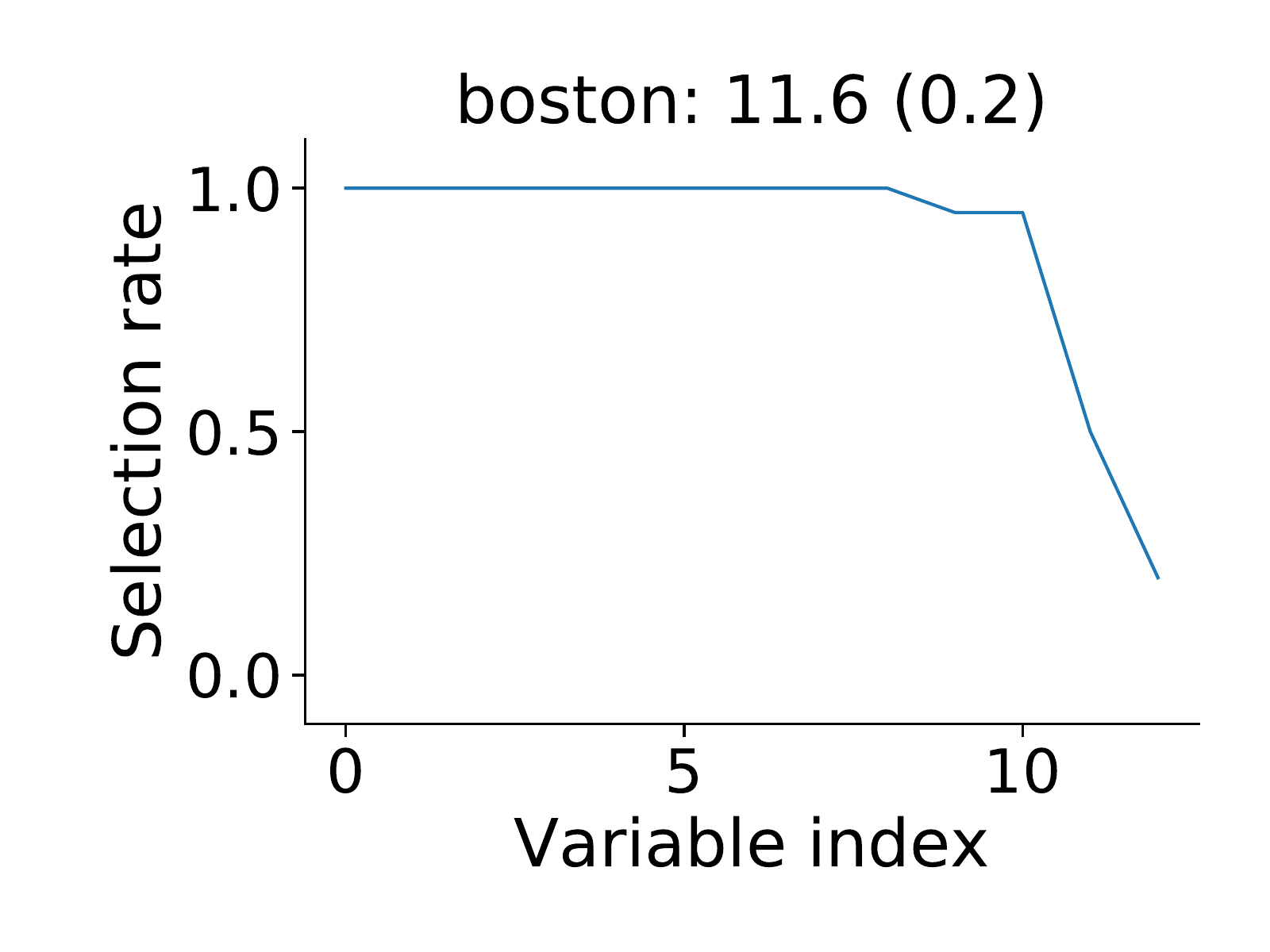}
\includegraphics[width=0.32\linewidth]{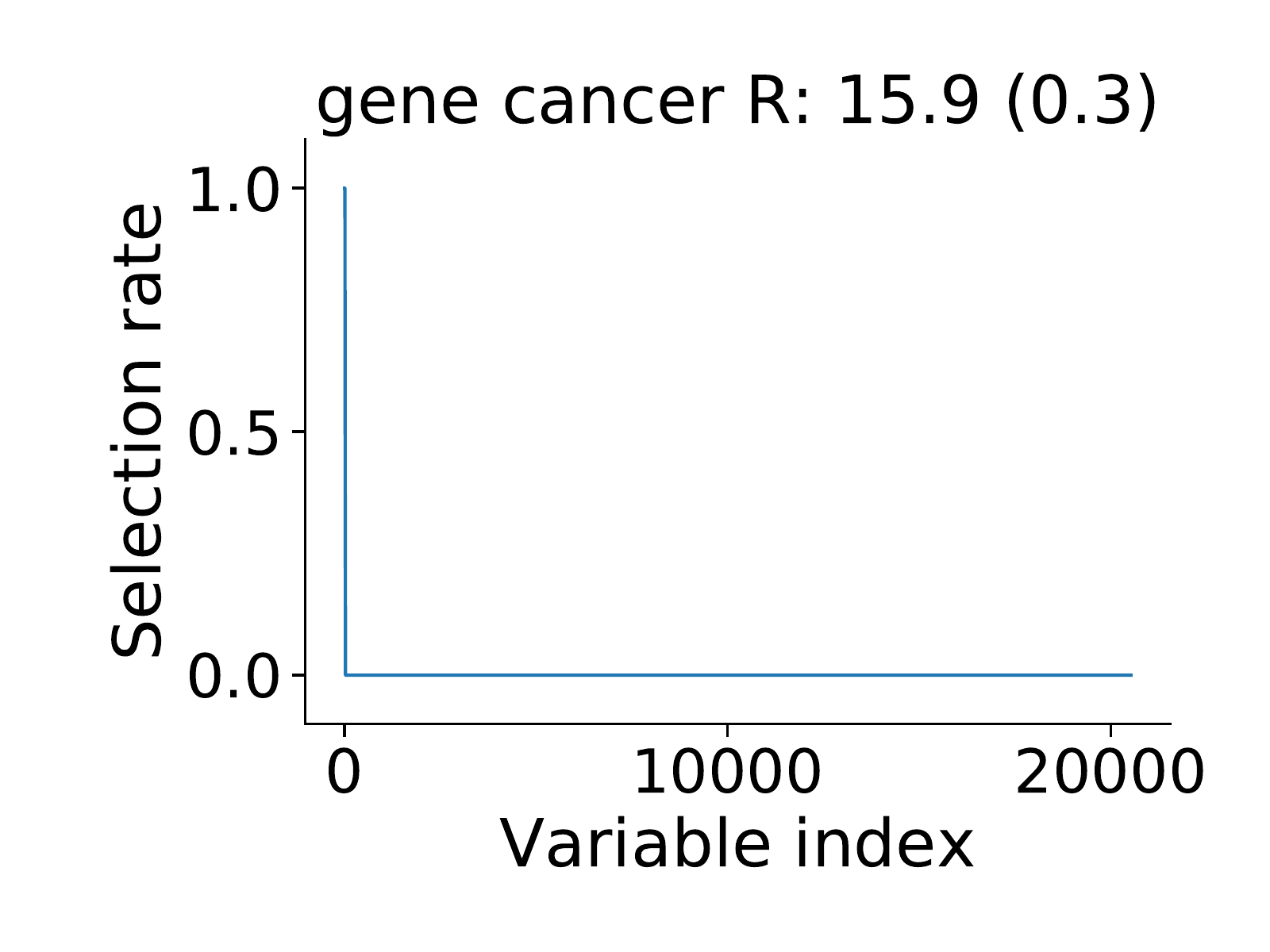}
\includegraphics[width=0.32\linewidth]{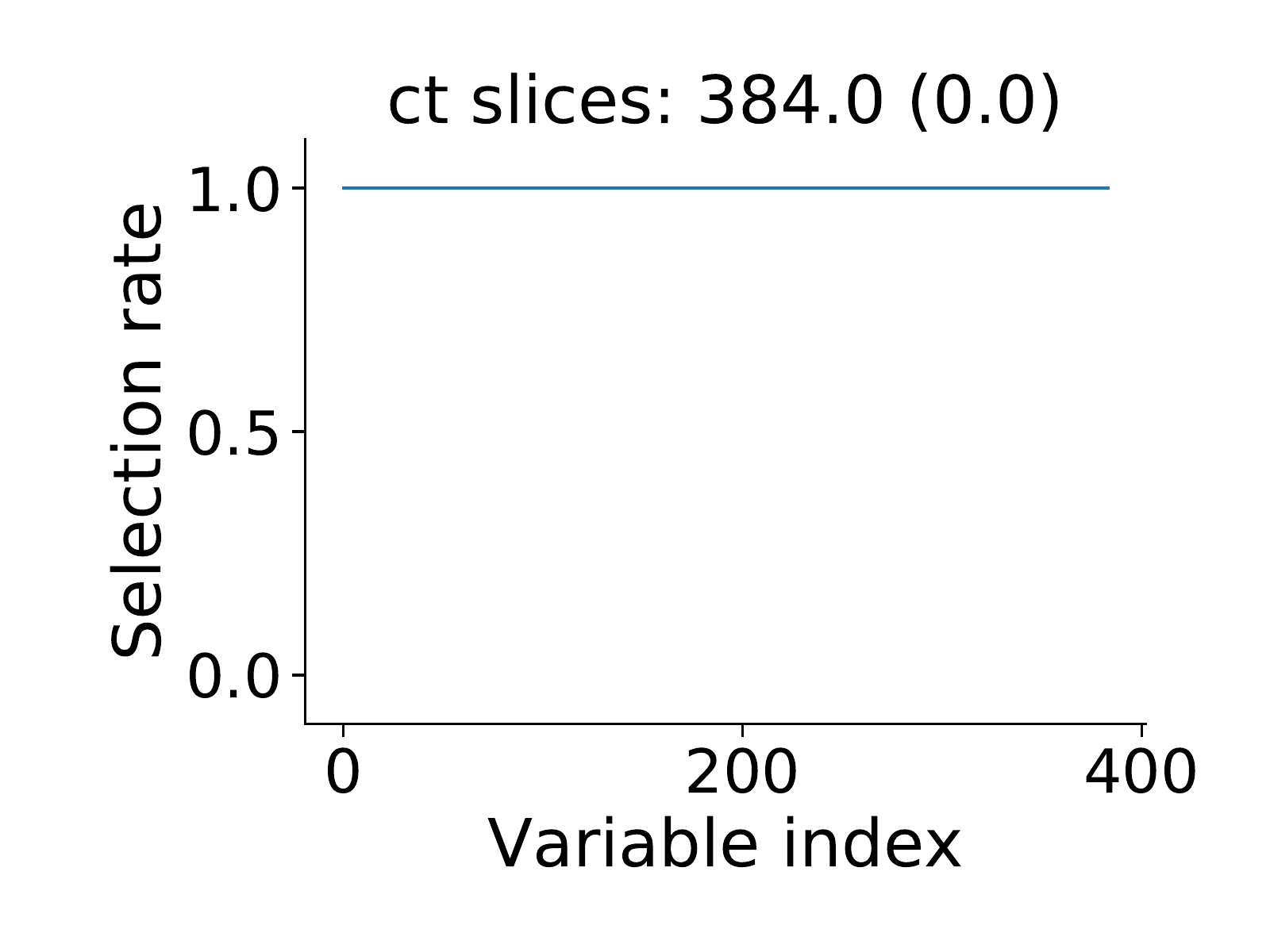}
\includegraphics[width=0.32\linewidth]{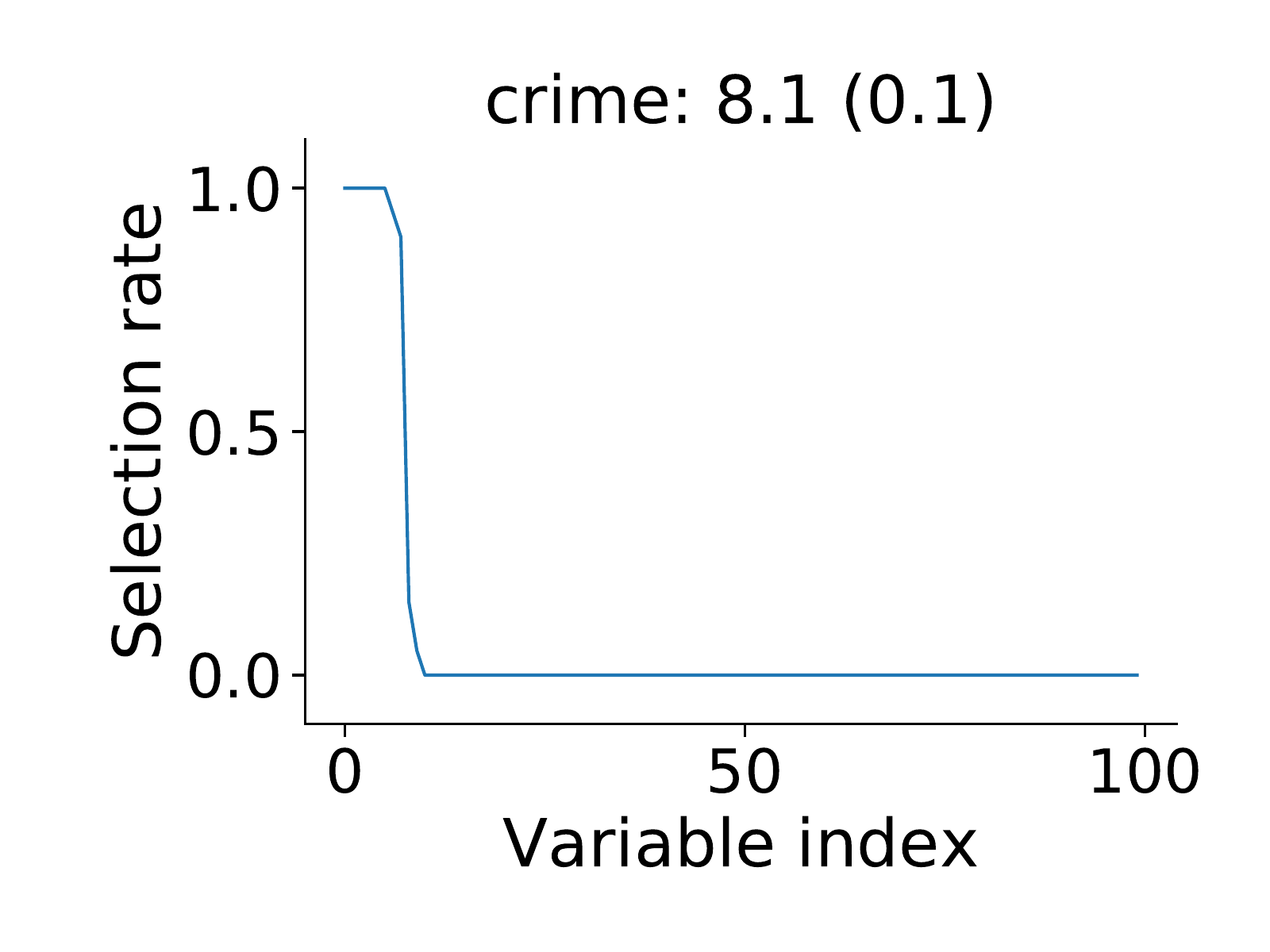}
\includegraphics[width=0.32\linewidth]{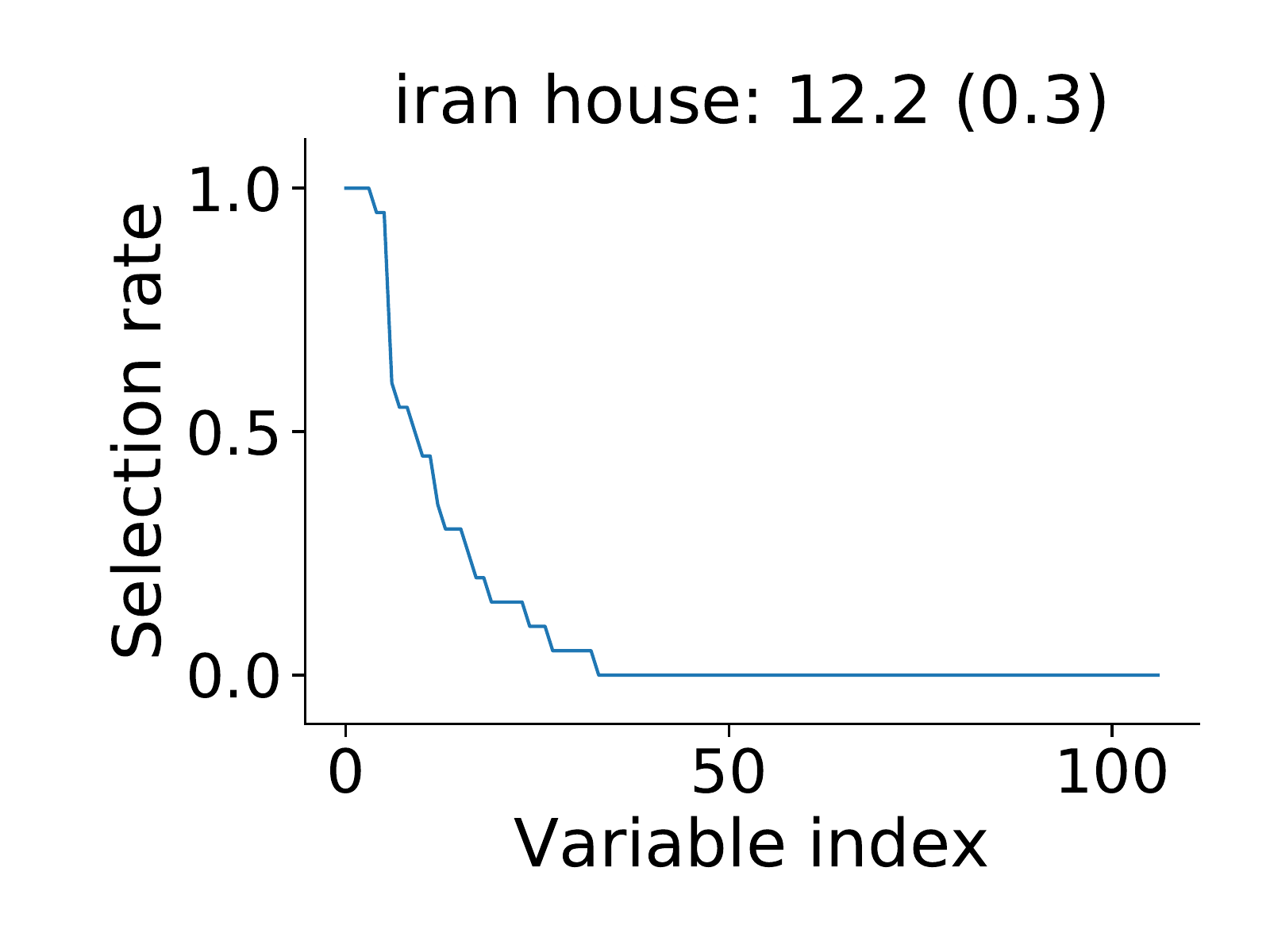}
\includegraphics[width=0.32\linewidth]{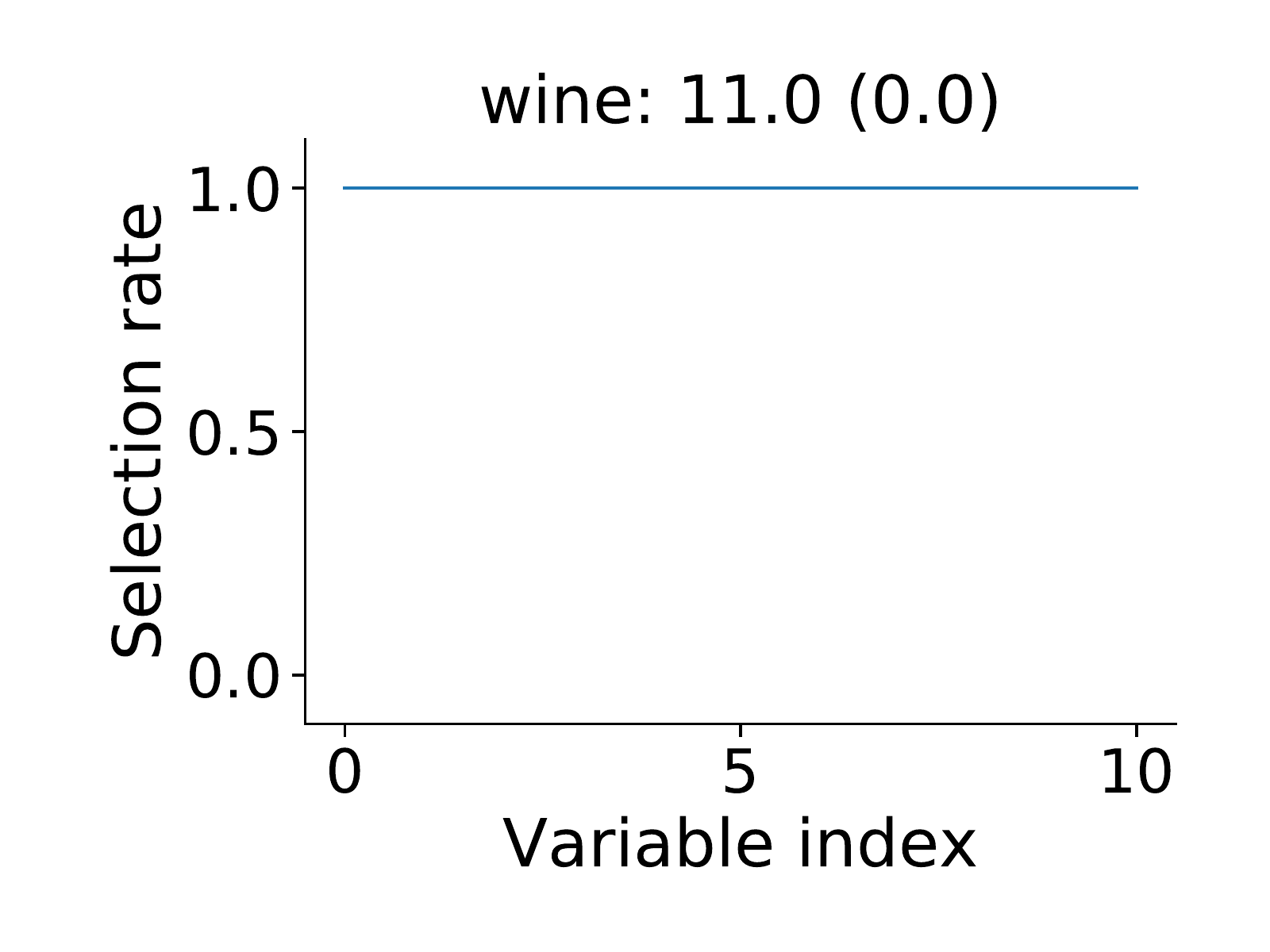}
\caption{Regression}
\label{fig:reg_support}
\end{subfigure}
\caption{
For each dataset, the proportion of member networks from EASIER-net that selected each variable, where each ensemble consists of 20 independently-trained networks.
Variables are sorted along the x-axis according to their selection rates.
The plot titles show the average support size and the standard error in parentheses.
The variable selection rates reflect the uncertainty of the true support.
}
\label{fig:uci_support}
\end{figure}

\section{Discussion}

We have shown that an ensemble of sparse-input hierarchical networks is a useful modeling technique for datasets with a small number of observations.
By employing skip-connections and sparsity-inducing penalties, this methodology can data-adaptively determine the appropriate model complexity: they can fit simple linear models, shallow neural networks, or deep networks.
Empirically, we find that EASIER-net can significantly outperform popular machine learning algorithms in terms of prediction accuracy.
Moreover, the model can do variable screening when the mutual information between the relevant and irrelevant covariates is low, and appears to induce a grouping effect when the mutual information is high.

The theoretical analyses of SIER-net in this paper provide probability bounds on the number of false negatives.
We have not bounded the number of false positives and leave this to future work.
Nevertheless, the empirical experiments demonstrate that both SIER-net and EASIER-net are able to perform support recovery under certain conditions.

Because neural networks are highly modular, our method can be easily extended.
For example, by modifying the output layer and objective function, we can output prediction sets/intervals to better quantify model uncertainty \citep{Taylor2000-iu, Feng2019-of}.
For specific problem domains, we can also tailor the network structure and sparsity pattern to reflect our prior knowledge, such as applying a group lasso to known groups of covariates \citep{Yuan2006-kz}.

Finally, interpretability of SIER-net and EASIER-net remains an issue, even with the help of sparsity.
To better understand the model's inner-workings, one may try using techniques like variable importance measures \citep{Ribeiro2016-jn, Lundberg2017-xl, Feng2018-ez}, saliency maps \citep{Simonyan2015-mo}, and influence functions \citep{Koh2017-fg}.

\section*{Acknowledgments}
The authors thank Frederick A. Matsen IV for helpful discussions and suggestions. This work was supported by NIH Grant DP5OD019820.

\appendix

\appendix
\section*{Appendix}

\section{Sparse-input hierarchical network definition}
\label{sec:define}
For the reader's convenience, the full set of equations that define a sparse-input hierarchical network $f_{W, b, \beta, \alpha}$ is given below:
\begin{align}
z_{1,i} &= \beta_i x_{i} & \forall i = 1,...,d\\
\boldsymbol{z}_{l} &= \phi(\boldsymbol{z}_{l - 1} W_{l-1} + \boldsymbol{b}_{l-1} ) & \forall l = 2,...,L - 1\\
\boldsymbol{\zeta}_l & = \boldsymbol{z}_{l } W'_{l} + \boldsymbol{b}'_{l} & \forall l = 1,...,L - 1 \\
f_{W, b, \beta, \alpha}(\boldsymbol{x}) &= \boldsymbol{z}_{L} = \phi_{\text{out}}\left(
\sum_{l=1}^{L-1} \frac{|\alpha_{l}|}{\sum_{l' = 1}^{L - 1}|\alpha_{l'}|} \boldsymbol{\zeta}_{l}
\right)
\label{eq:sier-net}
\end{align}

\section{Proofs}
\label{sec:proofs}

For convenience, we overload the notation as follows.
Consider a probability space $(\Omega, \mathcal{A}, P)$ with sample space $\Omega = \mathcal{X} \times \mathcal{Y}$, $\sigma$-algebra $\mathcal{A}$, and probability measure $P$.
For a real-valued function $f: \mathcal{X} \times \mathcal{Y} \mapsto \mathbb{R}$, let $P f(X, Y)$ denote the expected value of $f(X, Y)$.
For a set-valued function $g: \mathcal{X} \times \mathcal{Y} \mapsto \mathcal{A}$, let $P g(X, Y)$ denote the probability measure of $g(X, Y)$.

The proof for Theorem~\ref{thrm:upper} uses standard empirical process techniques.
Recall that for a given distribution $\mu$, the $\mu$-complexity of function class $\mathcal{F}$ for sample $(X_1,...,X_n)$ is defined as
\begin{align}
E\left[\sup_{f \in \mathcal{F}} \frac{1}{n} \sum_{i=1}^n \xi_i (f(X_i) - f^*(X_i)) \mid X_1,...,X_n \right]
\label{eq:complexity}
\end{align}
where $\xi \sim \mu$ are independent and identically distributed (IID).
The Gaussian complexity of $\mathcal{F}$ refers to the case when $\mu = N(0,1)$.
Let $\mathcal{G}_n(\mathcal{F})$ denote the uniform bound of the Gaussian complexity of $\mathcal{F}$ over all samples from $\mathcal{X}$.
The Rademacher complexity of $\mathcal{F}$ refers to the case when $\mu$ is the Rademacher distribution.
Analogously, let $\mathcal{R}_n(\mathcal{F})$ denote the uniform bound of the Rademacher complexity of $\mathcal{F}$ over all samples from $\mathcal{X}$.

We will use the $\mu$-complexity bounds derived in \citet{Barron2019-sb}.
We note that their bounds technically pertain to dense ReLU neural networks, not sparse-input hierarchical networks.
Nevertheless, it is easy to show that any sparse-input hierarchical network can be expressed as a dense ReLU neural network, albeit with additional hidden nodes per layer.
Thus their derived complexity bounds are also valid for sparse-input hierarchical networks.

\begin{proof}[Proof for Theorem~\ref{thrm:upper}]
	Let $P_n$ be the empirical distribution of the training data.
	Because $\hat{\theta}$ is a global empirical risk minimizer, we have that
	\begin{align*}
	P \left\{ (f_{\hat{\theta}_n}(X) -  Y)^2 - (f_{\theta^*}(X) - Y)^2 \right \}
	& \ \le
	\left(P - P_n \right )
	\left(
	(f_{\hat{\theta}_n}(X) -  Y)^2 - (f_{\theta^*}(X) - Y)^2
	\right)\\
	& \ =
	\left(P - P_n \right )
	\left(
	(f_{\hat{\theta}_n}(X) -  f_{\theta^*}(X))^2 - \epsilon (f_{\theta^*}(X) - Y)
	\right).
	\end{align*}
	Then for any $j = 1,\cdots,d$, we have that
	\begin{align}
	& P\left(
	|s - \hat{s}| \ge j
	\right)\\
	\le &
	P\left(
	P \left\{ (f_{\hat{\theta}_n}(X) - Y)^2 - (f_{\theta^*}(X) - Y)^2 \right \}
	\ge \gamma(j ; P)
	\right)\\
	\le&
	P\left(
	\left(P - P_n \right )
	\left(
	(f_{\hat{\theta}_n}(X) -  f_{\theta^*}(X))^2 - \epsilon (f_{\hat{\theta}_n}(X) - f_{\theta^*}(X))
	\right)
	\ge \gamma(j ; P)
	\right)\\
	\le&
	P\left(
	\left[
	\sup_{\theta \in \Theta_{L, B_0, B_1}}
	\left(P - P_n \right )
	(f(X) -  f_{\theta^*}(X))^2
	\right]
	+
	\left[
	\sup_{\theta \in \Theta_{L, B_0, B_1}}
	\left(P - P_n \right )
	\epsilon (f(X) - f_{\theta^*}(X))
	\right]
	\ge \gamma(j ; P)
	\right).
	\end{align}
	For convenience, let $\mathcal{F}_{L, B_0, B_1} = \{f_\theta:  \theta \in \Theta_{L, B_0, B_1}\}$.
	We can bound the above probability by the union bound
	\begin{align}
	P\left(
	|s - \hat{s}| \ge j
	\right)
	&\  \le
	P\left(
	4B_0 \mathcal{R}_n (\mathcal{F}_{L, B_0, B_1}) + \delta
	+
	\sigma\mathcal{G}_n (\mathcal{F}_{L, B_0, B_1}) + \delta
	\ge \gamma(j; P)
	\right)\\
	&\ + P\left(
	\sup_{f \in \mathcal{F}_{L, B_0, B_1}}
	\left(P - P_n \right )
	(f(X) -  f_{\theta^*}(X))^2
	- 4B_0 \mathcal{R}_n (\mathcal{F}_{L, B_0, B_1})
	\ge \delta
	\right)
	\label{eq:rademacher}
	\\
	&\ + P\left(
	\sup_{f \in \mathcal{F}_{L, B_0, B_1}}
	\left(P - P_n \right )
	\epsilon (f(X) - f_{\theta^*}(X))
	- \sigma \mathcal{G}_n (\mathcal{F}_{L, B_0, B_1})
	\ge \delta
	\right)
	\label{eq:gaussian}
	\end{align}

	So with minor modifications of the Rademacher complexity to handle squared of functions with infinity-norm bounded by $B_0$, we have that \eqref{eq:rademacher} is bounded by $\exp\left(-\frac{n \delta^2}{2 B_0^4} \right)$.
	Likewise, since $\epsilon$ is a mean-zero Gaussian RV with variance $\sigma^2$, then \eqref{eq:gaussian} is bounded by $2\exp\left(
	-\frac{n \delta^2}{8B_0^2 \sigma^2}
	\right)$.
	Combining the above inequalities, we have for all $j = 1,\cdots, d$ that
	\begin{align}
	P\left(
	|s - \hat{s}| \ge j
	\right)
	&\  \le
	\mathbbm{1}\left\{
	4B_0 \mathcal{R}_n (\mathcal{F}_{L, B_0, B_1}) + \sigma\mathcal{G}_n (\mathcal{F}_{L, B_0, B_1}) + 2\delta
	\ge \gamma(j ; P)
	\right \}
	\label{eq:complexities}
	\\
	& \ + \exp\left(
	-\frac{n \delta^2}{2 B_0^4}
	\right)\\
	& \ + 2\exp\left(
	-\frac{n \delta^2}{8B_0^2 \sigma^2}
	\right).
	\end{align}
	Finally, since \citet{Barron2019-sb} proved that $\mathcal{G}_n (\mathcal{F}_{L, B_0, B_1})$ and $\mathcal{R}_n (\mathcal{F}_{L, B_0, B_1})$ are both bounded by $B_1 \sqrt{2(L \log(2) + \log (2d))/n}$, we can plug this into \eqref{eq:complexities} and obtain our desired result.
\end{proof}

Our proof for Theorem~\ref{thrm:lower} depends on the following lemma, which is a straightforward adaptation of Proposition 15.1 in \citet{Wainwright2019-sp}.
For any positive integer $m \le \min(k, \lfloor d/2 \rfloor )$, select two distributions $P_{(1)}, P_{(2)} \in \mathcal{P}_{k, \gamma}$ such that
\begin{align}
\Delta(P_{(1)}, P_{(2)}) = \min(|s(\mu_{P_{(1)}}) \setminus s(\mu_{P_{(2)}})|, |s(\mu_{P_{(2)}}) \setminus s(\mu_{P_{(1)}})|) \ge 2m.
\label{eq:diff_support}
\end{align}
Let $Q$ denote the joint distribution over the pair of random variables $(Z, Y)$ generated using the following procedure:
\begin{enumerate}
	\item Sample $J$ from $\{1,2\}$ uniformly at random.
	\item Given  $J = j$, sample $X_1,....,X_n$ iid from $P_{(j)}$.
\end{enumerate}
Define a testing function $\psi$ as a mapping from $\mathcal{X}^n \mapsto \{1,2\}$.

\begin{lemma}
	For any positive integer $m \le \min(k, d/2)$, we have that
	\begin{align}
	\inf_{\hat{f}_n: |s(\hat{f}_n)| \le k}
	\sup_{P \in \mathcal{P}_{k, \gamma}}
	P\left(
	\left|
	s(\mu_{P}) \setminus s(\hat{f}_n)
	\right|
	\ge m
	\right)
	\ge \inf_{\psi} Q \left (\psi(X_1,....,X_n) \ne J \right )
	\end{align}
	where the infimum ranges over all test functions.
	\label{lemma:minimax}
\end{lemma}

\begin{proof}
	For any $M$, we have that
	\begin{align}
	\sup_{P \in \mathcal{P}_{k, \gamma}}
	P\left(
	\left|
	s(\mu_P) \setminus s(\hat{f}_n)
	\right|
	\ge m
	\right)
	\ge
	\frac{1}{2} \sum_{j=1}^2
	P_{(j)}\left( \left| s(\mu_{P_{(j)}}) \setminus s(\hat{f}_n) \right | \ge m \right).
	\end{align}
	Define the testing function $\psi(X_1,...,X_n) = \argmin_{j \in \{1,2\}} \left| s(\mu_{P_{(j)}}) \setminus s(\hat{f}_n) \right |$.
	We next show that if $\mu_{P_{(j)}}$ is the true distribution, the event $\left [\left| s(\mu_{P_{(j)}}) \setminus s(\hat{f}_n) \right | \le m \right ]$ implies that $\psi(X_1,...,X_n) = j$.
	Because $\hat{f}_n$ can choose at most $k$ elements in the support and the assumption that $\mu_{P_{(1)}}$ and $\mu_{P_{(2)}}$ differ by at least $2m$ elements in their support, we have that if $\left| s(\mu_{P_{(j)}}) \setminus s(\hat{f}_n) \right | \le m$ for some $j = 1,2$, then $\left| s(\mu_{P_{(j')}}) \setminus s(\hat{f}_n) \right | \ge m$ for $j' \ne j$.
	Moreover, this implies that $\psi(X_1,...,X_n) = j$.
	Thus, we have established that
	\begin{align}
	\frac{1}{2} \sum_{j=1}^2
	P_{(j)}\left( \left| s(\mu_{P_{(j)}}) \setminus s(\hat{f}_n) \right | \ge m \right)
	\ge Q\left (\psi(X_1,....,X_n) \ne J \right ).
	\end{align}
	Finally, take the infimum with respect to all estimators on the left hand side and the infimum over all induced tests on the right hand side.
	Since the full infimum can only be smaller, we have established the desired result.
\end{proof}

\begin{proof}[Proof for Theorem~\ref{thrm:lower}]
	From Le Cam's inequality, for any $P_{(1)}, P_{(2)} \in \mathcal{P}_{k, \gamma}$ that satisfy \eqref{eq:diff_support}, we have that
	\begin{align}
	\inf_{\psi} Q\left (\psi(X_1,....,X_n) \ne J \right ) \ge \frac{1}{2}
	\left( 1 - \left \| P_{(1)}^n - P_{(2)}^n \right \|_{TV} \right),
	\label{eq:le_cam}
	\end{align}
	where $\|\cdot\|_{TV}$ is the total variation norm for probability distributions.
	We then lower bound the right hand side by relating the KL-divergence to the total variation norm as follows:
	\begin{align}
	\left \| P_{(1)}^n - P_{(2)}^n \right \|_{TV} \le \sqrt{\frac{1}{2} D(P_{(1)}^n || P_{(2)}^n)} =\sqrt{\frac{n}{2} D(P_{(1)} || P_{(2)})}.
	\label{eq:kl_tv}
	\end{align}
	Since we assumed that $\epsilon \sim N(0,\sigma^2)$, the squared error loss $\ell$ is equal to the negative log likelihood scaled by $\sigma^2$.
	Thus, $\gamma(2m)/ \sigma^2$ is the minimum KL-divergence between two functions in $\mathcal{P}_{k, \gamma}$ with support differing by $2m$, i.e.
	\begin{align}
	\inf_{P_{(1)}, P_{(2)} \in \mathcal{P}_{k, \gamma}: \Delta(P_{(1)}, P_{(2)})\ge 2m } D(P_{(1)} || P_{(2)}) = \gamma(2m)/\sigma^2 .
	\label{eq:gamma_bound}
	\end{align}
	Combining the above results with Lemma~\ref{lemma:minimax}, we have that
	\begin{align}
	\inf_{\hat{f}_n: |s(\hat{f}_n)| \le k}
	\sup_{P \in \mathcal{P}_{k, \gamma}}
	P\left(
	\left|
	s(\mu_P) \setminus s(\hat{f}_n)
	\right|
	\ge m
	\right)
	& \ge
	\sup_{P_{(1)}, P_{(2)} \in \mathcal{P}_{k, \gamma}: \Delta(P_{(1)}, P_{(2)}) \ge 2m }
	\frac{1}{2}
	\left( 1 - \left \| P_{(1)}^n - P_{(2)}^n \right \|_{TV} \right)\\
	& \ge
	\frac{1}{2} \left(
	1 - \sqrt{\frac{n}{2} \gamma(2m)/\sigma^2 }
	\right),
	\end{align}
	where the first inequality follows from taking the supremum over the right hand side of \eqref{eq:le_cam}.
\end{proof}

\section{Hyperparameter default values}
\label{sec:hyperparam}

We recommend selecting the size of the ensemble $B$ to be sufficiently large for its predicted value to plateau.
We found that $B = 20$ worked well in all our experiments.
We initialize the sparse-input hierarchical networks to be sufficiently wide and deep to obtain a small training loss.
In the empirical analyses, we use 5 hidden layers and 100 hidden nodes per layer.
To perform penalized empirical minimization, we run Adam with the default learning rates and parameters until convergence and then batch proximal gradient descent until convergence.
We use a mini-batch size that is one third of the total size of the data.

To speed up K-fold cross-validation, one could try to fit a smaller ensemble for each candidate penalty parameter set.
In our work, we used $B = 10$ for tuning the penalty parameters.

\begin{table}
	\centering
	{\small
		\begin{tabular}{ccccc}
			Dataset & \# features & \# observations & \# classes & Held-out proportion\\
			\toprule
			\multicolumn{5}{c}{\textit{Classification}}\\
			soybean & 35 & 307 & 19 & 1/4 \\
			arrythmia & 279 & 452 & 13& 1/4 \\
			gene cancer C & 20531 & 801 & 5& 1/3 \\
			hill valley & 100 & 606 & 2 & 1/3\\
			semeion & 256 & 1593 & 10 & 1/3\\
			\midrule
			\multicolumn{5}{c}{\textit{Regression}}\\
			boston & 13 & 506 & --& 1/3\\
			gene cancer R & 20530 & 801 & -- & 1/3\\
			CT slices & 384 & 53500& --& 1/3\\
			crime & 122 & 1994 & --& 1/3\\
			Iran house & 103 & 372 & --& 1/3\\
			wine & 11 & 4898 & --& 1/3
		\end{tabular}
	}
	\caption{
		Summary statistics for the selected datasets from the UCI Machine Learning Repository.
		The datasets were chosen to represent varying dataset shapes and sizes.
		One third of the data was held out for testing unless a random split resulted in no samples from a particular class.
		The gene cancer regression (gene cancer R) task is a derivative of the gene cancer classification (gene cancer C) task, where we try to predict the expression value for the first gene in the dataset instead of the type of cancer.
	}
	\label{table:dataset_descrips}
\end{table}

\bibliographystyle{plainnat}
\bibliography{main}

\begin{thebibliography}{54}
\providecommand{\natexlab}[1]{#1}
\providecommand{\url}[1]{\texttt{#1}}
\expandafter\ifx\csname urlstyle\endcsname\relax
  \providecommand{\doi}[1]{doi: #1}\else
  \providecommand{\doi}{doi: \begingroup \urlstyle{rm}\Url}\fi

\bibitem[Abadi et~al.(2016)Abadi, Agarwal, Barham, Brevdo, Chen, Citro,
  Corrado, Davis, Dean, Devin, Ghemawat, Goodfellow, Harp, Irving, Isard, Jia,
  Jozefowicz, Kaiser, Kudlur, Levenberg, Mane, Monga, Moore, Murray, Olah,
  Schuster, Shlens, Steiner, Sutskever, Talwar, Tucker, Vanhoucke, Vasudevan,
  Viegas, Vinyals, Warden, Wattenberg, Wicke, Yu, and Zheng]{Abadi2016-zh}
Mart{\'\i}n Abadi, Ashish Agarwal, Paul Barham, Eugene Brevdo, Zhifeng Chen,
  Craig Citro, Greg~S Corrado, Andy Davis, Jeffrey Dean, Matthieu Devin, Sanjay
  Ghemawat, Ian Goodfellow, Andrew Harp, Geoffrey Irving, Michael Isard,
  Yangqing Jia, Rafal Jozefowicz, Lukasz Kaiser, Manjunath Kudlur, Josh
  Levenberg, Dan Mane, Rajat Monga, Sherry Moore, Derek Murray, Chris Olah,
  Mike Schuster, Jonathon Shlens, Benoit Steiner, Ilya Sutskever, Kunal Talwar,
  Paul Tucker, Vincent Vanhoucke, Vijay Vasudevan, Fernanda Viegas, Oriol
  Vinyals, Pete Warden, Martin Wattenberg, Martin Wicke, Yuan Yu, and Xiaoqiang
  Zheng.
\newblock {TensorFlow}: {Large-Scale} machine learning on heterogeneous
  distributed systems.
\newblock March 2016.

\bibitem[Bardsley et~al.(2014)Bardsley, Solonen, Haario, and
  Laine]{Bardsley2014-bp}
Johnathan~M Bardsley, Antti Solonen, Heikki Haario, and Marko Laine.
\newblock {Randomize-Then-Optimize}: A method for sampling from posterior
  distributions in nonlinear inverse problems.
\newblock \emph{SIAM J. Sci. Comput.}, 36\penalty0 (4):\penalty0 A1895--A1910,
  January 2014.
\newblock ISSN 1064-8275.
\newblock \doi{10.1137/140964023}.

\bibitem[Barron and Klusowski(2019)]{Barron2019-sb}
Andrew~R Barron and Jason~M Klusowski.
\newblock Complexity, statistical risk, and metric entropy of deep nets using
  total path variation.
\newblock \emph{arXiv}, February 2019.

\bibitem[Bengio(2012)]{Bengio2012-zq}
Yoshua Bengio.
\newblock Practical recommendations for {Gradient-Based} training of deep
  architectures.
\newblock In Gr{\'e}goire Montavon, Genevi{\`e}ve~B Orr, and Klaus-Robert
  M{\"u}ller, editors, \emph{Neural Networks: Tricks of the Trade: Second
  Edition}, pages 437--478. Springer Berlin Heidelberg, Berlin, Heidelberg,
  2012.
\newblock ISBN 9783642352898.
\newblock \doi{10.1007/978-3-642-35289-8\_26}.

\bibitem[Bergstra and Bengio(2012)]{Bergstra2012-ct}
James Bergstra and Yoshua Bengio.
\newblock Random search for {Hyper-Parameter} optimization.
\newblock \emph{J. Mach. Learn. Res.}, 13\penalty0 (Feb):\penalty0 281--305,
  2012.
\newblock ISSN 1532-4435, 1533-7928.

\bibitem[Breiman(1996)]{Breiman1996-bf}
Leo Breiman.
\newblock Bagging predictors.
\newblock \emph{Mach. Learn.}, 24\penalty0 (2):\penalty0 123--140, August 1996.
\newblock ISSN 0885-6125, 1573-0565.
\newblock \doi{10.1023/A:1018054314350}.

\bibitem[B{\"u}hlmann and van~de Geer(2011)]{Buhlmann2011-ko}
Peter B{\"u}hlmann and Sara van~de Geer.
\newblock \emph{Statistics for {High-Dimensional} Data: Methods, Theory and
  Applications}.
\newblock Springer, Berlin, Heidelberg, 2011.
\newblock ISBN 9783642201912.
\newblock \doi{10.1007/978-3-642-20192-9}.

\bibitem[Dua and Graff(2017)]{Dua:2019}
Dheeru Dua and Casey Graff.
\newblock {UCI} machine learning repository.
\newblock \emph{arXiv}, 2017.
\newblock URL \url{http://archive.ics.uci.edu/ml}.

\bibitem[Feng and Simon(2019)]{Feng2019-nw}
Jean Feng and Noah Simon.
\newblock {Sparse-Input} neural networks for high-dimensional nonparametric
  regression and classification.
\newblock \emph{arXiv}, 2019.

\bibitem[Feng et~al.(2018)Feng, Williamson, Simon, and Carone]{Feng2018-ez}
Jean Feng, Brian Williamson, Noah Simon, and Marco Carone.
\newblock Nonparametric variable importance using an augmented neural network
  with multi-task learning.
\newblock \emph{ICML}, 80:\penalty0 1496--1505, 2018.

\bibitem[Feng et~al.(2019)Feng, Sondhi, Perry, and Simon]{Feng2019-of}
Jean Feng, Arjun Sondhi, Jessica Perry, and Noah Simon.
\newblock Selective prediction-set models with coverage guarantees.
\newblock \emph{arXiv}, June 2019.

\bibitem[Glorot et~al.(2011)Glorot, Bordes, and Bengio]{Glorot2011-pd}
Xavier Glorot, Antoine Bordes, and Yoshua Bengio.
\newblock Deep sparse rectifier neural networks.
\newblock \emph{International Conference on Artificial Intelligence and
  Statistics}, 2011.

\bibitem[Goodfellow et~al.(2016)Goodfellow, Bengio, and
  Courville]{Goodfellow2016-sw}
Ian Goodfellow, Yoshua Bengio, and Aaron Courville.
\newblock \emph{Deep Learning}.
\newblock MIT Press, November 2016.
\newblock ISBN 9780262337373.

\bibitem[Guvenir et~al.(1997)Guvenir, Acar, Demiroz, and Cekin]{Guvenir1997-uk}
H~A Guvenir, B~Acar, G~Demiroz, and A~Cekin.
\newblock A supervised machine learning algorithm for arrhythmia analysis.
\newblock In \emph{Computers in Cardiology 1997}, pages 433--436.
  ieeexplore.ieee.org, September 1997.
\newblock \doi{10.1109/CIC.1997.647926}.

\bibitem[Guyon et~al.(2005)Guyon, Gunn, Ben-Hur, and Dror]{Guyon2005-fp}
Isabelle Guyon, Steve Gunn, Asa Ben-Hur, and Gideon Dror.
\newblock Result analysis of the {NIPS} 2003 feature selection challenge.
\newblock In L~K Saul, Y~Weiss, and L~Bottou, editors, \emph{Advances in Neural
  Information Processing Systems 17}, pages 545--552. MIT Press, 2005.

\bibitem[Han et~al.(2015)Han, Pool, Tran, and Dally]{Han2015-wz}
Song Han, Jeff Pool, John Tran, and William Dally.
\newblock Learning both weights and connections for efficient neural network.
\newblock In C~Cortes, N~D Lawrence, D~D Lee, M~Sugiyama, and R~Garnett,
  editors, \emph{Advances in Neural Information Processing Systems 28}, pages
  1135--1143. Curran Associates, Inc., 2015.

\bibitem[Hanson and Pratt(1989)]{Hanson1989-cj}
Stephen~Jose Hanson and Lorien~Y Pratt.
\newblock Comparing biases for minimal network construction with
  {Back-Propagation}.
\newblock In D~S Touretzky, editor, \emph{Advances in Neural Information
  Processing Systems 1}, pages 177--185. Morgan-Kaufmann, 1989.

\bibitem[He et~al.(2016)He, Zhang, Ren, and Sun]{He2016-qz}
K~He, X~Zhang, S~Ren, and J~Sun.
\newblock Deep residual learning for image recognition.
\newblock In \emph{2016 {IEEE} Conference on Computer Vision and Pattern
  Recognition ({CVPR})}, pages 770--778, June 2016.
\newblock \doi{10.1109/CVPR.2016.90}.

\bibitem[Hoeting et~al.(1999)Hoeting, Madigan, Raftery, and
  Volinsky]{Hoeting1999-pg}
Jennifer~A Hoeting, David Madigan, Adrian~E Raftery, and Chris~T Volinsky.
\newblock Bayesian model averaging: a tutorial.
\newblock \emph{Stat. Sci.}, 14\penalty0 (4):\penalty0 382--417, November 1999.
\newblock ISSN 0883-4237, 2168-8745.
\newblock \doi{10.1214/ss/1009212519}.

\bibitem[Huang et~al.(2017)Huang, Liu, d.~Maaten, and Weinberger]{Huang2017-bt}
G~Huang, Z~Liu, L~v d.~Maaten, and K~Q Weinberger.
\newblock Densely connected convolutional networks.
\newblock In \emph{2017 {IEEE} Conference on Computer Vision and Pattern
  Recognition ({CVPR})}, pages 2261--2269, July 2017.
\newblock \doi{10.1109/CVPR.2017.243}.

\bibitem[Kingma and Ba(2015)]{Kingma2015-oy}
Diederik~P Kingma and Jimmy Ba.
\newblock Adam: A method for stochastic optimization.
\newblock \emph{International Conference for Learning Representations}, 2015.

\bibitem[Koh and Liang(2017)]{Koh2017-fg}
Pang~Wei Koh and Percy Liang.
\newblock Understanding black-box predictions via influence functions.
\newblock \emph{Proceedings of the 34th International Conference on Machine
  Learning}, 2017.

\bibitem[Lakshminarayanan et~al.(2017)Lakshminarayanan, Pritzel, and
  Blundell]{Lakshminarayanan2017-wn}
Balaji Lakshminarayanan, Alexander Pritzel, and Charles Blundell.
\newblock Simple and scalable predictive uncertainty estimation using deep
  ensembles.
\newblock In I~Guyon, U~V Luxburg, S~Bengio, H~Wallach, R~Fergus,
  S~Vishwanathan, and R~Garnett, editors, \emph{Advances in Neural Information
  Processing Systems 30}, pages 6402--6413. Curran Associates, Inc., 2017.

\bibitem[LeCun et~al.(1990)LeCun, Denker, and Solla]{LeCun1990-ng}
Yann LeCun, John~S Denker, and Sara~A Solla.
\newblock Optimal brain damage.
\newblock In D~S Touretzky, editor, \emph{Advances in Neural Information
  Processing Systems 2}, pages 598--605. Morgan-Kaufmann, 1990.

\bibitem[Liang et~al.(2018)Liang, Li, and Zhou]{Liang2018-ev}
Faming Liang, Qizhai Li, and Lei Zhou.
\newblock Bayesian neural networks for selection of drug sensitive genes.
\newblock \emph{J. Am. Stat. Assoc.}, 113\penalty0 (523):\penalty0 955--972,
  July 2018.
\newblock ISSN 0162-1459.
\newblock \doi{10.1080/01621459.2017.1409122}.

\bibitem[Louizos et~al.(2018)Louizos, Welling, and Kingma]{Louizos2018-mx}
Christos Louizos, Max Welling, and Diederik~P Kingma.
\newblock Learning sparse neural networks through l\_0 regularization.
\newblock \emph{International Conference on Learning Representations}, February
  2018.

\bibitem[Lu and Van~Roy(2017)]{Lu2017-kz}
Xiuyuan Lu and Benjamin Van~Roy.
\newblock Ensemble sampling.
\newblock \emph{Advances in Neural Information Processing Systems}, pages
  3258--3266, 2017.

\bibitem[Lundberg and Lee(2017)]{Lundberg2017-xl}
Scott~M Lundberg and Su-In Lee.
\newblock A unified approach to interpreting model predictions.
\newblock \emph{Advances in Neural Information Processing Systems 30}, pages
  4765--4774, 2017.

\bibitem[MacKay(1996)]{MacKay1996-ey}
David J~C MacKay.
\newblock Bayesian {Non-Linear} modeling for the prediction competition.
\newblock In Glenn~R Heidbreder, editor, \emph{Maximum Entropy and Bayesian
  Methods: Santa Barbara, California, {U.S.A.}, 1993}, pages 221--234. Springer
  Netherlands, Dordrecht, 1996.
\newblock ISBN 9789401587297.
\newblock \doi{10.1007/978-94-015-8729-7\_18}.

\bibitem[Madigan and Raftery(1994)]{Madigan1994-ni}
David Madigan and Adrian~E Raftery.
\newblock Model selection and accounting for model uncertainty in graphical
  models using occam's window.
\newblock \emph{J. Am. Stat. Assoc.}, 89\penalty0 (428):\penalty0 1535--1546,
  1994.
\newblock ISSN 0162-1459.
\newblock \doi{10.2307/2291017}.

\bibitem[Mandt et~al.(2017)Mandt, Hoffman, and Blei]{Mandt2017-xk}
Stephan Mandt, Matthew~D Hoffman, and David~M Blei.
\newblock Stochastic gradient descent as approximate bayesian inference.
\newblock \emph{J. Mach. Learn. Res.}, January 2017.
\newblock ISSN 1532-4435.

\bibitem[Neal(1996)]{Neal1996-at}
Radford~M Neal.
\newblock \emph{Bayesian Learning for Neural Networks}.
\newblock Lecture Notes in Statistics. 1996.
\newblock ISBN 9780387947242.
\newblock \doi{10.1007/978-1-4612-0745-0}.

\bibitem[Neal and Zhang(2006)]{Neal2006-jb}
Radford~M Neal and Jianguo Zhang.
\newblock High dimensional classification with bayesian neural networks and
  dirichlet diffusion trees.
\newblock In Isabelle Guyon, Masoud Nikravesh, Steve Gunn, and Lotfi~A Zadeh,
  editors, \emph{Feature Extraction: Foundations and Applications}, pages
  265--296. Springer Berlin Heidelberg, Berlin, Heidelberg, 2006.
\newblock ISBN 9783540354888.
\newblock \doi{10.1007/978-3-540-35488-8\_11}.

\bibitem[Olson et~al.(2018)Olson, Wyner, and Berk]{Olson2018-vs}
Matthew Olson, Abraham Wyner, and Richard Berk.
\newblock Modern neural networks generalize on small data sets.
\newblock In S~Bengio, H~Wallach, H~Larochelle, K~Grauman, N~Cesa-Bianchi, and
  R~Garnett, editors, \emph{Advances in Neural Information Processing Systems
  31}, pages 3619--3628. Curran Associates, Inc., 2018.

\bibitem[Parikh and Boyd(2014)]{Parikh2014-kd}
Neal Parikh and Stephen Boyd.
\newblock Proximal algorithms.
\newblock \emph{Foundations and Trends\textregistered{} in Optimization},
  1\penalty0 (3):\penalty0 127--239, 2014.
\newblock ISSN 2167-3888.
\newblock \doi{10.1561/2400000003}.

\bibitem[Paszke et~al.(2019)Paszke, Gross, Massa, Lerer, Bradbury, Chanan,
  Killeen, Lin, Gimelshein, Antiga, Desmaison, Kopf, Yang, DeVito, Raison,
  Tejani, Chilamkurthy, Steiner, Fang, Bai, and Chintala]{Paszke2019-tz}
Adam Paszke, Sam Gross, Francisco Massa, Adam Lerer, James Bradbury, Gregory
  Chanan, Trevor Killeen, Zeming Lin, Natalia Gimelshein, Luca Antiga, Alban
  Desmaison, Andreas Kopf, Edward Yang, Zachary DeVito, Martin Raison, Alykhan
  Tejani, Sasank Chilamkurthy, Benoit Steiner, Lu~Fang, Junjie Bai, and Soumith
  Chintala.
\newblock {PyTorch}: An imperative style, {High-Performance} deep learning
  library.
\newblock In \emph{Advances in Neural Information Processing Systems 32}, pages
  8026--8037. Curran Associates, Inc., 2019.

\bibitem[Pearce et~al.(2020)Pearce, Leibfried, Brintrup, Zaki, and
  Neely]{Pearce2020-ee}
Tim Pearce, Felix Leibfried, Alexandra Brintrup, Mohamed Zaki, and Andy Neely.
\newblock Uncertainty in neural networks: Approximately bayesian ensembling.
\newblock \emph{International Conference on Artificial Intelligence and
  Statistics}, 2020.

\bibitem[Raftery et~al.(1997)Raftery, Madigan, and Hoeting]{Raftery1997-pr}
Adrian~E Raftery, David Madigan, and Jennifer~A Hoeting.
\newblock Bayesian model averaging for linear regression models.
\newblock \emph{J. Am. Stat. Assoc.}, 92\penalty0 (437):\penalty0 179--191,
  March 1997.
\newblock ISSN 0162-1459.
\newblock \doi{10.1080/01621459.1997.10473615}.

\bibitem[Ribeiro et~al.(2016)Ribeiro, Singh, and Guestrin]{Ribeiro2016-jn}
Marco~Tulio Ribeiro, Sameer Singh, and Carlos Guestrin.
\newblock `` why should i trust you?'' explaining the predictions of any
  classifier.
\newblock In \emph{Proceedings of the 22nd {ACM} {SIGKDD} international
  conference on knowledge discovery and data mining}, pages 1135--1144.
  dl.acm.org, 2016.

\bibitem[Ripley and Hjort(1996)]{Ripley1996-ub}
Brian~D Ripley and N~L Hjort.
\newblock \emph{Pattern Recognition and Neural Networks}.
\newblock Cambridge University Press, January 1996.
\newblock ISBN 9780521460866.

\bibitem[Scardapane et~al.(2017)Scardapane, Comminiello, Hussain, and
  Uncini]{Scardapane2017-yj}
Simone Scardapane, Danilo Comminiello, Amir Hussain, and Aurelio Uncini.
\newblock Group sparse regularization for deep neural networks.
\newblock \emph{Neurocomputing}, 241:\penalty0 81--89, June 2017.
\newblock ISSN 0925-2312.
\newblock \doi{10.1016/j.neucom.2017.02.029}.

\bibitem[Simon et~al.(2013)Simon, Friedman, Hastie, and
  Tibshirani]{Simon2013-tq}
Noah Simon, Jerome Friedman, Trevor Hastie, and Robert Tibshirani.
\newblock A {Sparse-Group} lasso.
\newblock \emph{J. Comput. Graph. Stat.}, 22\penalty0 (2):\penalty0 231--245,
  April 2013.
\newblock ISSN 1061-8600.
\newblock \doi{10.1080/10618600.2012.681250}.

\bibitem[Simonyan and Zisserman(2015)]{Simonyan2015-mo}
Karen Simonyan and Andrew Zisserman.
\newblock Very deep convolutional networks for {Large-Scale} image recognition.
\newblock \emph{International Conference on Learning Representations}, 2015.

\bibitem[Snoek et~al.(2012)Snoek, Larochelle, and Adams]{Snoek2012-zd}
Jasper Snoek, Hugo Larochelle, and Ryan~P Adams.
\newblock Practical bayesian optimization of machine learning algorithms.
\newblock In F~Pereira, C~J~C Burges, L~Bottou, and K~Q Weinberger, editors,
  \emph{Advances in Neural Information Processing Systems 25}, pages
  2951--2959. Curran Associates, Inc., 2012.

\bibitem[Taylor(2000)]{Taylor2000-iu}
James~W Taylor.
\newblock A quantile regression neural network approach to estimating the
  conditional density of multiperiod returns.
\newblock \emph{J. Forecast.}, 19\penalty0 (4):\penalty0 299--311, 2000.
\newblock ISSN 0277-6693.

\bibitem[Viallefont et~al.(2001)Viallefont, Raftery, and
  Richardson]{Viallefont2001-wa}
V~Viallefont, A~E Raftery, and S~Richardson.
\newblock Variable selection and bayesian model averaging in case-control
  studies.
\newblock \emph{Stat. Med.}, 20\penalty0 (21):\penalty0 3215--3230, November
  2001.
\newblock ISSN 0277-6715.
\newblock \doi{10.1002/sim.976}.

\bibitem[Wainwright(2019)]{Wainwright2019-sp}
Martin~J Wainwright.
\newblock \emph{{High-Dimensional} Statistics: A {Non-Asymptotic} Viewpoint}.
\newblock Cambridge University Press, February 2019.
\newblock ISBN 9781108498029.

\bibitem[Wen et~al.(2016)Wen, Wu, Wang, Chen, and Li]{Wen2016-lo}
Wei Wen, Chunpeng Wu, Yandan Wang, Yiran Chen, and Hai Li.
\newblock Learning structured sparsity in deep neural networks.
\newblock In D~D Lee, M~Sugiyama, U~V Luxburg, I~Guyon, and R~Garnett, editors,
  \emph{Advances in Neural Information Processing Systems 29}, pages
  2074--2082. Curran Associates, Inc., 2016.

\bibitem[Wilson and Izmailov(2020)]{Wilson2020-jw}
Andrew~Gordon Wilson and Pavel Izmailov.
\newblock Bayesian deep learning and a probabilistic perspective of
  generalization.
\newblock \emph{arXiv}, February 2020.

\bibitem[Yoon and Hwang(2017)]{Yoon2017-ao}
Jaehong Yoon and Sung~Ju Hwang.
\newblock Combined group and exclusive sparsity for deep neural networks.
\newblock In Doina Precup and Yee~Whye Teh, editors, \emph{Proceedings of the
  34th International Conference on Machine Learning}, volume~70 of
  \emph{Proceedings of Machine Learning Research}, pages 3958--3966,
  International Convention Centre, Sydney, Australia, 2017. PMLR.

\bibitem[Yu et~al.(2012)Yu, Seide, Li, and Deng]{Yu2012-ka}
D~Yu, F~Seide, G~Li, and L~Deng.
\newblock Exploiting sparseness in deep neural networks for large vocabulary
  speech recognition.
\newblock In \emph{2012 {IEEE} International Conference on Acoustics, Speech
  and Signal Processing ({ICASSP})}, pages 4409--4412, March 2012.
\newblock \doi{10.1109/ICASSP.2012.6288897}.

\bibitem[Yuan and Lin(2006)]{Yuan2006-kz}
Ming Yuan and Yi~Lin.
\newblock Model selection and estimation in regression with grouped variables.
\newblock \emph{J. R. Stat. Soc. Series B Stat. Methodol.}, 68\penalty0
  (1):\penalty0 49--67, 2006.
\newblock ISSN 1369-7412.

\bibitem[Zhang et~al.(2017)Zhang, Bengio, Hardt, Recht, and
  Vinyals]{Zhang2017-ws}
Chiyuan Zhang, Samy Bengio, Moritz Hardt, Benjamin Recht, and Oriol Vinyals.
\newblock Understanding deep learning requires rethinking generalization.
\newblock \emph{International Conference on Learning Representations}, 2017.

\bibitem[Zou and Hastie(2005)]{Zou2005-wm}
Hui Zou and Trevor Hastie.
\newblock Regularization and variable selection via the elastic net.
\newblock \emph{J. R. Stat. Soc. Series B Stat. Methodol.}, 67\penalty0
  (2):\penalty0 301--320, 2005.
\newblock ISSN 1369-7412, 1467-9868.

\end{thebibliography}


\begin{thebibliography}{2}
\providecommand{\natexlab}[1]{#1}
\providecommand{\url}[1]{\texttt{#1}}
\expandafter\ifx\csname urlstyle\endcsname\relax
  \providecommand{\doi}[1]{doi: #1}\else
  \providecommand{\doi}{doi: \begingroup \urlstyle{rm}\Url}\fi

\bibitem[Barron and Klusowski(2019)]{Barron2019-sb}
Andrew~R Barron and Jason~M Klusowski.
\newblock Complexity, statistical risk, and metric entropy of deep nets using
  total path variation.
\newblock \emph{arXiv}, February 2019.

\bibitem[Wainwright(2019)]{Wainwright2019-sp}
Martin~J Wainwright.
\newblock \emph{{High-Dimensional} Statistics: A {Non-Asymptotic} Viewpoint}.
\newblock Cambridge University Press, February 2019.
\newblock ISBN 9781108498029.

\end{thebibliography}
\end{document}